\documentclass[twoside]{article}

\usepackage[accepted]{aistats2021}
\usepackage{soul}
\usepackage{url}
\usepackage[colorlinks, linkcolor=blue, citecolor=blue]{hyperref}
\usepackage[small]{caption}
\usepackage{graphicx}
\usepackage{amsmath}
\usepackage{xfrac}
\usepackage{booktabs}
\usepackage{algorithm}
\usepackage{algorithmic}
\usepackage{subcaption}

\usepackage{circledsteps}
\usepackage{bbm}
\usepackage{bm}
\usepackage{tabularx} 
\usepackage{graphicx} 
\usepackage{natbib}
\usepackage{tcolorbox}
\usepackage{amsmath}
\usepackage{amsthm}
\usepackage{multirow}
\usepackage{amssymb}  
\usepackage{tabularx} 
\usepackage{amsmath}  
\usepackage{mathtools}

\newtheorem{theorem}{Theorem}

\newtheorem{corollary}{Corollary}
\newtheorem{lemma}{Lemma}
\newtheorem{assumption}{Assumption}

\newtheorem{definition}{Definition}
\newtheorem{remark}{Remark}

\DeclareMathOperator*{\argmax}{arg\,max}
\usepackage{mathtools}
\DeclarePairedDelimiter{\ceil}{\lceil}{\rceil}

\newcommand{\etal}{\textit{et al.}}

\newcommand{\bH}{{\bf H}}
\newcommand{\bX}{{\bf X}}
\newcommand{\bV}{{\bf V}}

\newcommand{\bI}{{\bf I}}
\newcommand{\bK}{{\bf K}}
\newcommand{\bG}{{\bf G}}

\newcommand{\bN}{{\bf N}}
\newcommand{\bA}{{\bf A}}
\newcommand{\bB}{{\bf B}}

\newcommand{\bS}{{\bf \Sigma}}

\newcommand{\x}{{\bf x}}

\newcommand{\bu}{{\bf u}}
\newcommand{\bz}{{\bf z}}
\newcommand{\bh}{{\bf h}}

\newcommand{\by}{{\bf y}}
\newcommand{\bn}{{\bf n}}
\newcommand{\bk}{{\bf k}}

\newcommand{\bomega}{{\bm \omega}}

\newcommand{\balpha}{{\bm \alpha}}
\newcommand{\bvarepsilon}{{\bm \varepsilon}}

\newcommand{\bupsilon}{{\bm \upsilon}}
\newcommand{\bGamma}{{\bf \Gamma}}
\newcommand{\btheta}{{\bm \theta}}

\newcommand{\bbP}{\mathbb P}
\newcommand{\bbR}{\mathbb R}
\newcommand{\bbE}{\mathbb E}

\newcommand{\cX}{\mathcal X}
\newcommand{\cH}{\mathcal H}
\newcommand{\cR}{\mathcal R}
\newcommand{\cB}{\mathcal B}
\newcommand{\cD}{\mathcal D}
\newcommand{\cO}{\mathcal O}

\newcommand{\cI}{\mathcal I}
\newcommand{\cF}{\mathcal F}
\newcommand{\cN}{\mathcal N}
\newcommand{\cY}{\mathcal Y}
\newcommand{\cS}{\mathcal S}

\newcommand{\cA}{\mathcal A}
\newcommand{\cT}{\mathcal T}
\newcommand{\cZ}{\mathcal Z}

\newcommand{\fR}{\mathfrak R}
\def\doappendix{1}

\begin{document}

%
\runningtitle{No-Regret Algorithms for Private Gaussian Process Bandit Optimization}

%

\twocolumn[

\aistatstitle{No-Regret Algorithms for \\ Private Gaussian Process Bandit Optimization}

\aistatsauthor{Abhimanyu Dubey}

\aistatsaddress{Media Lab and Institute for Data, Systems and Society \\ Massachusetts Institute of Technology \\ \texttt{dubeya@mit.edu}} ]

\begin{abstract}
The widespread proliferation of data-driven
decision-making has ushered in a recent interest in the design of privacy-preserving algorithms. In this paper, we consider the ubiquitous problem of gaussian process (GP) bandit optimization from the lens of privacy-preserving statistics. We propose a solution for differentially private GP bandit optimization that combines a uniform kernel approximator with random perturbations, providing a generic framework to create differentially-private (DP) Gaussian process bandit algorithms. For two specific DP settings - joint and local differential privacy, we provide algorithms based on efficient quadrature Fourier feature approximators, that are computationally efficient and provably no-regret for popular stationary kernel functions. Our algorithms maintain differential privacy throughout the optimization procedure and critically do not rely explicitly on the sample path for prediction, making the parameters straightforward to release as well.
\end{abstract}
\section{Introduction}
Gaussian Process (GP) bandit optimization~\citep{srinivas2010gaussian} is a sequential decision problem that has a variety of human-centered applications, e.g., clinical drug trials~\citep{costabal2019machine, park2013bayesian, peterson2017personalized}, personalized shopping recommendations~\citep{rohde2018recogym, zhou2019leverage}, news feed ranking~\citep{agarwal2018online, letham2019bayesian, vanchinathan2014explore}. It is increasingly becoming desirable that algorithms interacting with such data maintain the privacy of the individuals whose information is used~\citep{cummings2018role}. 

GP bandit optimization involves learning a function $f$ via repeated interaction in rounds. At any round $t = 1, 2, ...$, the learner is presented with a \textit{decision set} $\cD_t \subset \bbR^d$ from which it must select an action $\x_t$ and obtain a random reward $y_t = f(\x_t) + \varepsilon_t$. The algorithm selects actions in order to minimize regret $\cR(T) = \sum_{t} [\max_{\x \in \cD_t} f(\x) - f(\x_t)]$. Algorithm design is focused on minimizing \textit{pseudoregret} $\bbE[\cR(T)]$. In deployment settings, each round corresponds to selecting a random user $i_t$. The decision set $\cD_t$ is a representation of the user' behavior and $y_t$ refers to the users response to $\x_t$. In this case, privacy refers to privacy with respect to both $(\cD_t, y_t)$~\citep{shariff2018differentially}.

Provably no-regret algorithms with differential privacy have been proposed for multi-armed bandits~\citep{tossou2015algorithms, mishra2015nearly}, linear contextual bandits~\citep{shariff2018differentially, agarwal2017price} and tabular RL~\citep{vietri2020private}. For GP optimization, however, the problem is more challenging. Most applications assume $f$ to lie in a (potentially) infinite-dimensional reproducing kernel Hilbert space (RKHS), and standard techniques for introducing privacy are inapplicable due to the~\textit{curse of dimensionality}~\citep{liu2017dimension, meeds2014gps}: the posterior mean and variance for these methods require storing the sample path $(\x_t, y_t)_t$, and are $\Omega(t)$ to evaluate. Moreover, as the learnt function itself is dependent on the sample path (containing sensitive data), privatized release of the function is also a challenge~\citep{smith2016differentially}. In this paper, we propose algorithms that guarantee differential privacy with respect to continual observation during optimization, and also the private release of learnt parameters.

\textbf{Contributions}. First, we propose a generic framework (and regret bound) for GP bandits that utilizes a finite-dimensional $\epsilon-$uniform approximation of infinte-dimensional kernels and integrates random perturbations to the GP posterior, allowing for various no-regret private GP algorithms based on the kernel approximation method and privacy guarantee required.

Next, In the joint differentially private (JDP) setting (Defn.~\ref{def:jdp}), we propose a novel \texttt{GP-UCB} algorithm (Alg.~\ref{alg:main_algo}) for stationary kernels admitting a decomposable Fourier transform (Assumption \ref{ass:decomposability}) that satisfies $(\alpha, \beta)$-JDP while obtaining $\widetilde\cO(\sqrt{T\gamma_T/\alpha})$\footnote{$\gamma_T$ is the \textit{maximum information gain}, see Definition~\ref{def:mig}.} pseudoregret. This bound matches (up to logarithmic factors) the lower bound for isotropic kernels~\citep{scarlett2017lower}, and admits an identical dependence on $\alpha$ as linear bandits~\citep{shariff2018differentially}. Thirdly, inspired by the recent interest in locally DP methods~\citep{bebensee2019local}, we present a stronger variant of JDP dubbed locally-joint differential privacy (Defn.~\ref{def:ldp}) for sequential decision-making that imposes constraints on  each user's data separately. We propose an algorithm that achieves $(\alpha, \beta)-$local JDP with $\widetilde\cO(T^{\sfrac{3}{4}}\sqrt{\gamma_T/\alpha})$ pseudoregret. We conjecture that the constraints from local JDP necessitate the $\cO(T^{\sfrac{1}{4}})$ departure from typical near-optimal regret (Remark~\ref{rem:suboptimal_ldp_regret}).

Our approach can be coarsely summarized with two steps - we first project $f$ from its (infinite-dimesional) RKHS into a finite-dimensional approximating RKHS, following which, we directly perturb the posterior mean and variance of the resulting GP (in the approximating space) to ensure privacy without the curse of dimensionality, providing \textit{provably no-regret} algorithms for private GP bandit optimization. Our approach additionally avoids the parameter release problem~\citep{smith2016differentially, kusner2015differentially} since we do not explicitly store the sample path for prediction, and rely instead on cumulative sums (Remark~\ref{rem:param_release}).

\textbf{Organization}. We first discuss crucial related work and introduce necessary notation and preliminaries, subsequent to which we introduce our general framework for \texttt{GP-UCB} using noisy approximate features. We discuss quadrature Fourier features and present our algorithm and its associated regret bounds. Next, we discuss the two models of privacy studied, and present privacy mechanisms. We defer proofs to the appendix, and present concise proof sketches in the main paper. 

\section{Related Work}
\textbf{Gaussian Process Bandits}. Gaussian Processes~\citep{williams2006gaussian} have been widely used for the bandit optimization of unknown functions in an RKHS. The seminal work of \citet{srinivas2010gaussian} introduced the nonparameteric \texttt{GP-UCB} algorithm, that introduced contextual-bandit style confidence bounds for optimisation in infinite-dimensional RKHSes. A variant of the \textit{expected improvement} decision rule~\citep{movckus1975bayesian} was proposed via the \texttt{GP-EI} algorithm~\citep{snoek2012practical}. By a stronger martingale analysis,~\citet{chowdhury2017kernelized} achieve the \texttt{IGP-UCB} algorithm, that improves \texttt{GP-UCB} regret by a factor of $\cO(\ln^{\sfrac{3}{2}}T)$. For a family of isotropic squared-exponential $d$-dimensional kernels, \citet{scarlett2017lower} establish lower bounds on the achievable regret of $\Omega(\sqrt{T(\log T)^{d+2}})$, which matches (ignoring polylogarithmic factors) the $\widetilde\cO(\sqrt{T})$ rate achieved by \texttt{IGP-UCB} and \texttt{GP-UCB}. Our work relies on the research in approximate methods for kernel approximation, which has seen a lot of recent interest. The seminal work of ~\citet{rahimi2008random} proposed random Fourier features (RFF) by a Monte-Carlo approximation of the Fourier basis, with additional work establishing finite-sample convergence rates~\citep{avron2017random}. We propose a noisy variant of the more efficient quadrature Fourier features (QFF)~\citep{munkhoeva2018quadrature} that have been previously employed in GP optimization with success~\citep{mutny2018efficient}. An alternative approach based on sampling fewer points from the algorithm's history based on matrix sketching has been proposed in \citet{calandriello2019gaussian}.

\textbf{Differentially-Private Bandit Learning}. Differentially private (DP) methods for bandit optimisation have received significant attention recently. For the multi-armed bandit case, UCB and Thompson sampling algorithms have been proposed for pure-DP~\citep{mishra2015nearly}, with subsequent improvements~\citep{tossou2015algorithms}. For the contextual linear bandit,~\citet{shariff2018differentially} introduce an algorithm that utilizes matrix perturbations that our work effectively generalizes to infinite-dimensional stationary GPs. Note that this algorithm is inapplicable for general GPs as it assumes that the features are finite-dimensional. See~\citet{basu2020differential} for a summary of regret bounds for private multi-armed bandits. For Gaussian process bandits and Bayesian Optimisation (BO), \citet{kusner2015differentially} consider the problem of \textit{releasing} GP parameters \textit{after} optimization under differential privacy constraints, by analysing the sensitivity of the final parameters. Our work handles a more challenging setting, where parameters must be private \textit{throughout} the optimisation process. An application of DP to the Gaussian process regression problem was studied in the work of \citet{smith2016differentially}, however with no regret guarantees.
\section{Preliminaries}
\textbf{Notation}. We denote vectors by lowercase solid characters, i.e., $\x$ and matrices by uppercase solid characters $\bX$. We denote the $\bS-$ellipsoid norm of a vector $\x$ as $\lVert x \rVert_\bS = \sqrt{\x^\top\bS\x}$, that a symmetric matrix $\bA$ is PSD by $\bA \succcurlyeq {\bf 0}$, and the L\"owner ordering of symmetric PSD matrices by $\bA \succcurlyeq \bB$, which implies $\bA - \bB \succcurlyeq {\bf 0}$. 

\textbf{GP Bandit Optimization}. We consider the problem of sequential reward maximization under a fixed but unknown reward function $f: \cD \rightarrow \bbR$ over a (potentially infinite) set of actions (arms) $\cD \subset \bbR^d$. The problem proceeds in rounds $t = 1, 2, ..., T$ where, in each round, the objective is to select an action $\x_t \in \cD_t$ and obtain a reward $y_t = f(\x_t) + \varepsilon_t$ such that the cumulative reward $\sum_{t \in [T]} y_t$ is maximized depending on the history $(\x_\tau, y_\tau)_{\tau < t}$, and $\varepsilon_t$ is sampled from a zero-mean sub-Gaussian distribution with parameter $\lambda$. Gaussian Process (GP) modeling proposes to use a Gaussian likelihood model for observations and a GP prior for the uncertainty over $f$. A Gaussian Process (GP) over $\cD$, denoted by $\text{GP}(\mu(\cdot), k(\cdot, \cdot))$ is a collection of random variables $(f(\x))_{\x \in \cD}$ such that every finite subset of variables $(f(\x_\tau))_{\tau=1}^t$ is jointly Gaussian with mean $\bbE[f(\x_\tau)] = \mu(\x_\tau)$ and covariance $\bbE[(f(\x_\tau) - \mu(\x_\tau))(f(\x_{\tau'}) - \mu(\x_{\tau'}))] = k(\x_\tau, \x_{\tau'}),\ \tau,\tau' \in [t]$ where $k(\cdot, \cdot)$ is the kernel function associated with the reproducing kernel Hilbert space (RKHS) $\cH_k(\cD)$ in which we assume $f$ has norm at most $B$, i.e., $\lVert f\rVert_k \leq B$. We use an initial prior distribution $\text{GP}(0, \rho^2 k(\cdot,\cdot))$ for some $\rho > 0$. Consequently it is also assumed that the noise samples $\varepsilon_t$ are drawn from $\cN(0, \lambda\rho^2)$\footnote{The algorithm only requires $\varepsilon_t$ to be $\lambda$-sub-Gaussian, i.e., the \textit{agnostic} setting \citep{srinivas2010gaussian}.}. We then obtain that the observed samples $\by_t = (y_\tau)_{\tau < t}$ and $f(\x)$ are jointly Gaussian given $\bX_t = (\x_\tau)_{\tau < t}$,
\begin{equation}
    \begin{bmatrix}
    f(\x) \\
    \by_t 
    \end{bmatrix}\sim\cN\left({\bf 0},
    \begin{bmatrix}
    \rho^2 k(\x, \x) & \rho^2 \bk_t(\x)^\top \\
    \rho^2 \bk_t(\x) & \rho^2(\bK_t + \lambda\bI)
    \end{bmatrix}
    \right).
\end{equation}
Where $\bK_t = (k(\x_\tau, \x_{\tau'}))_{\tau, \tau'}^{t, t}$ is the matrix of kernel evaluations at time $t$, and $\bk_t(\x) = [k(\x_1, \x), ..., k(\x_t, \x)]^\top$ is the vector of kernel evaluations of any input $\x$. Conditioned on $(\x_\tau, y_\tau)_{\tau < t}$, the posterior mean and variance of $f$ is given as,
\begin{align}
\label{eqn:original_update}
    \mu_t(\x) &= k_t(\x)^\top(\bK_t + \lambda \bI)^{-1}\by_t,\\
    \sigma^2_t(\x) &= \left(k(\x, \x) - \bk_t(\x)^\top(\bK_t + \lambda \bI)^{-1}\bk_t(\x)\right).
\end{align}
The kernel $k(\cdot, \cdot)$ additionally admits a representation in terms of its feature space $\Phi$ such that $k(\x, \x') = \Phi(\x)^\top\Phi(\x)$, where $\Phi : \bbR^d \rightarrow \bbR^m$ is the feature embedding. This provides an alternative representation of the posterior mean and variance,
\begin{align} 
\label{eqn:update}
    \mu_t(\x) &= (\bS_t + \lambda\bI)^{-1}\Phi(\bX)^\top\by_t,\\
    \sigma^2_t(\x) &= \rho^2\Phi(\x)^\top(\bS_t + \lambda\bI)^{-1}\Phi(\x), \text{ for}
\end{align}
\begin{equation*}
    \bS_t = \Phi(\bX_t)^\top\Phi(\bX_t), \Phi(\bX_t) = [\Phi(\x_1)^\top, ..., \Phi(\x_{t-1})^\top]^\top.
\end{equation*}
$\Phi$ can potentially be infinite-dimensional (e.g., for squared-exponential $k$), and hence this representation is not applicable to many popular kernel families. The regret achieved by existing algorithms depends on the \textit{maximum information gain}, a quantity that depends on the covariance structure of the feature space.

\begin{definition}[Information Gain~\citep{srinivas2010gaussian}]
\label{def:mig}
For $y_t = f(\x_t) + \varepsilon_t$, let $A \subset \mathcal X$ be a finite subset such that $|A| = T$. Let $\by_A = {\bf f}_A + \bm\varepsilon_A$ where ${\bf f}_A = (f(\bm x_i))_{\bm x_i \in A}$ and $\bm\varepsilon_A \sim \mathcal N(0, \rho^2)$. The information gain is $\gamma_T \triangleq \max_{A \subset \mathcal X : |A|=T} H(\by_A) - H(\by_A | f)$, where $H(\cdot)$ is the entropy of a random variable. For linear $k$, $\gamma_T = \mathcal O(d\log T)$. For RBF $k$, $\gamma_T = \mathcal O((\log T)^{d+1})$. For Mat\'ern $k$ with $\nu > 1$, $\gamma_T = \mathcal O(T^{\frac{d(d+1)}{2\nu + d(d+1)}}(\log T))$.
\end{definition}

\textbf{Differential Privacy} (DP). Differential Privacy~\citep{dwork2014algorithmic} is a cryptographically secure framework to introduce privacy, widely prevalent in machine learning. Let algorithm $\cA: \cX \rightarrow \cY$, let $X, X' \in \cX$ operate on samples from $\cX$ producing outputs in $\cY$. An algorithm is $(\alpha, \beta)-$differentially private if for any two inputs $X, X' \in \cX$ that differ in only one entry and any $\cS \subset \cY$, 
\begin{align}
\label{eqn:dp}
    \bbP(\cA(X) \in \cS) \leq e^{\alpha}\bbP(\cA(X') \in \cS) + \beta.
\end{align}
In the continual observation setting of sequential decision-making, this would imply that the algorithm be private with respect to all values $(\x_\tau, y_\tau)_{\tau=1}^T$ at each $t \in [T]$. However, as demonstrated in \citet{shariff2018differentially}, any algorithm DP with respect to $(\x_t, y_t)$ at the instance $t$ provably incurs $\Omega(T)$ regret. Therefore we adopt the notion of \textit{joint} differential privacy, which does not require privacy with respect to the inputs $(\x_t, y_t)$ at each instant $t \in [T]$ (Section~\ref{sec:privacy_jdp}). We additionally consider the stronger notion of \textit{locally} joint DP, which additionally requires that the algorithm cannot access $(\x_\tau, y_\tau)_{\tau < t}$ directly (Section~\ref{sec:privacy_ldp}).

\section{Noisy Proximal Features \& \texttt{GP-UCB}}
The primary challenge in creating differentially-private algorithms for \textit{bandit estimation} in arbitrary RKHSes is the curse of dimensionality - the two central quantities $\mu_t$ and $\sigma^2_t$ both require the point-wise kernel evaluations $(\bk_t(\x))$ and the kernel Gram matrix $(\bK_t)$ at all times, potentially requiring $\cO(\sqrt{t})$ noise in order to preserve privacy. In this paper, we tackle this hurdle by optimizing $f$ under a surrogate RKHS $\cF_m$ that of finite dimension $m$ instead of the original (potentially infinite-dimensional) RKHS $\cH_k$. To ensure a reasonable bound on the regret, we require that $\cF_m$ approximates $\cH_k$ closely, as formalized below.

\begin{definition}[Uniform Approximation]
Let $k : \cD \times \cD \rightarrow \bbR$ be a stationary kernel with associated RKHS $\cH_k$, and $\Phi : \cD \rightarrow \bbR^m$. Then $\Phi$ $\epsilon$-uniformly approximates $k$ iff $\sup_{\x, \x' \in \cD} |k(\x, \x') - \Phi(\x)^\top\Phi(\x)| \leq \epsilon$. The corresponding \textbf{approximating space} defined by $\Phi$ is given by $\cF_m(\Phi) \triangleq \left\{f(\cdot) = \btheta^\top\Phi(\cdot) | \btheta \in \bbR^m \right\}$.
\end{definition}
Therefore, if $\cF_m$ (resp. $\Phi$) can approximate $\cH_k$ without many features, one can devise an \textit{approximate} Gaussian process algorithm directly using $\Phi$.
\begin{align}
    \bG_t = \Phi(\bX_t)^\top\Phi(\bX_t) + \lambda\bI, \bu_t = \bG_t^{-1}\Phi(\bX_t)^\top\by_t.
\end{align}
These parameters allow us to obtain the posterior mean $\mu_t(\x) = \bu_t^\top\Phi(\x)$ and variance $\sigma^2_t(\x) = \rho^2\lVert \Phi(\x) \rVert_{\bG_t^{-1}}^2$. However, these parameters are obviously not differentially private with respect to the sequences $(\bX_t, \by_t)$. An efficient way to achieve privacy is to ensure that at each instant $t$, $(\bG_t, \bu_t)$ are differentially-private with respect to the sequence $(\x_\tau, y_\tau)_{\tau < t}$~\citep{shariff2018differentially}. This can be achieved by carefully perturbing $(\bG_t, \bu_t)$ with random noise $(\bH_t, \bh_t)$ to create differentially-private parameters. While the exact form of $\bH_t, \bh_t$ will be specified by the nature of privacy (see Section~\ref{sec:privacy}), we can represent a variety of noise models by spectral bounds, summarized by the following abstraction.
\begin{definition}[Spectral Bounds on Noise]
\label{def:accurate}
For a sequence of perturbations $(\bH_t)_{t=1}^T$ and $(\bh_t)_{t=1}^T$, the bounds $0 < \lambda_{\min} \leq \lambda_{\max}$ are $(\zeta/2T)-$accurate if with probability at least $1-\zeta/2T$, for each $t$ in $[T]$:
\begin{align*}
    \lVert \bH_t \rVert \leq \lambda_{\max}, \lVert \bH_t^{-1} \rVert \leq 1/\lambda_{\min}, \lVert \bh_t \rVert_{\bH_t^{-1}} \leq \kappa.
\end{align*}
\end{definition}
Let us use the shorthand $\bG_t = \bS_t + \lambda\bI$, where $\bS_t = \Phi(\bX_t)^\top\Phi(\bX_t)$. The perturbed $\bS_t$ and $\bu_t$ are given as $\widetilde{\bS}_t = \bS_t + \bH_t, \widetilde{\bu}_t = \bu_t + \bh_t$ for any sequence ($\bH_t, \bh_t$).
\subsection{\texttt{GP-UCB} with Noisy Proximal Features}
Our algorithm is built on the \texttt{GP-UCB} algorithm~\citep{srinivas2010gaussian} that constructs a confidence ellipsoid around the posterior $\widetilde\mu_t$ such that the function $f$ lies within the confidence ellipsoid with high probability. The key observation, is that we do not need to optimize for $f$ directly. Given an $\epsilon$-uniformly approximating feature $\Phi$ (resp. $\cF_m$), then the following result guarantees the existence of a function close to $f$ in $\cF_m$.
\begin{lemma}[Existence of Proximal Space~(Lemma 4 of \citet{mutny2018efficient})]
\label{lem:existence_proximal_space}
Let $k$ be a kernel defining the RKHS $\cH_k$ and $f \in \cH_k$, such that the spectral characteristic function is bounded by $B$. Assuming that the defining points of $f$ come from the set $\cD$, let $\cF_m$ be an approximating space with a mapping $\Phi$ such that this mapping is an $\epsilon$-approximation to the kernel $k$. Then there exists $\widehat\mu \in \cF_m$ (with corresponding feature $\widehat\btheta$ such that $\widehat\mu(\x) = \langle \widehat\btheta, \Phi(\x) \rangle$), such that $\sup_{\x \in \cD} \left| \widehat\mu(\x) - f(\x) \right| \leq B\epsilon$. 
\end{lemma}
Lemma~\ref{lem:existence_proximal_space} implies that there exists a fixed point $\widehat\mu \in \cF_m$ such that $\sup_{\x \in \cD} \left| \widehat\mu(\x) - f(\x) \right| \leq B\epsilon$. This implies that the regret incurred at any instant $t$ when optimizing for $f$ is at most $B\epsilon$ larger than the regret obtained when optimizing for $\widehat\mu$. We therefore optimize directly in the surrogate space $\cF_m$ to learn $\widehat\mu$. \texttt{GP-UCB} with noisy approximate features selects, for a sequence $(\beta_t)_{t=1}^T$, the action $\x_t \in \cD_t$ determined as:
\begin{equation}
\resizebox{0.65\linewidth}{!}{$
    \x_t = \argmax_{\x \in \cD_t} \widetilde\mu_t(\x) + \beta^{1/2}_t\widetilde\sigma_t(\x).
    $
}
\end{equation}
The sequence $(\beta_t)_{t=1}^T$ is chosen such that $\widetilde\mu_t(\x)$ is close to $\widehat\mu(\x)$ with high probability. To accomplish this, we present the central result as follows.
\begin{theorem}[$\beta_t$ concentration]
\label{thm:beta}
Let $\lambda_{\min}, \lambda_{\max}$ and $\kappa$ be $(\zeta/2T)$
-accurate and regularizers $\bH_t \succcurlyeq 0 \ \forall t \in [T]$ are PSD. Let $\widehat\mu$ be a function in the RKHS $\cF_m$ that $\epsilon$-approximates $f \in \cH_k$ (Lemma~\ref{lem:existence_proximal_space}). Then, with probability at least $1-\zeta/2$, for any $\x \in \cD$ we have for each $t \in [T]$ simultaneously,
\begin{multline*}
    \left|\widehat\mu(\x) - \widetilde\mu_t(\x)\right| \leq \widetilde\sigma_t(\x)\Bigg(B\sqrt{\frac{\lambda_{\max}}{\rho^2} + 1} \\+ \frac{tB\epsilon}{\rho\sqrt{\lambda_{\min}}}  + \frac{\kappa}{\rho}+\sqrt{\log\det\left(\frac{\widetilde{\bS}_t + \lambda\bI}{\lambda+\lambda_{\min}}\right)+2\ln\frac{2}{\zeta}} \Bigg).
\end{multline*}
The sequence $(\beta_t^{\sfrac{1}{2}})_{t=1}^T$ is chosen as the multiplicative factor of $\widetilde\sigma_t(\x)$, i.e., $\left|\widehat\mu(\x) - \widetilde\mu_t(\x)\right| \leq \beta_t^{\sfrac{1}{2}}\widetilde\sigma_t(\x)$.
\end{theorem}
\begin{algorithm}[t]
\caption{\textsc{Approximate} \texttt{GP-UCB}}
\small
\label{alg:main_algo}
\begin{algorithmic} 
\STATE \textbf{Input}: $m, \Phi$ that $\epsilon$-uniformly approximates $k$.
\STATE \textsc{Privatizer }\textbf{Initialize}: $\bS_1 = {\bf 0}, \bu_1 = {\bm 0}$.
\FOR{round $t=1, 2, ..., T$}
\STATE \underline{\textsc{Server}}:
\STATE Receive $\cD_t$ from environment.
\STATE Receive $\widetilde{\bS}_t = \bS_t + \bH_t, \widetilde\bu_t = \bu_t + \bh_t \leftarrow$\textsc{Privatizer}.
\STATE Set $\bV_t \leftarrow \widetilde{\bS}_t + \lambda\bI, \widetilde\btheta_t \leftarrow \bV_t^{-1}\widetilde\bu_t$.
\STATE Compute $\beta_t$ based on Theorem~\ref{thm:beta}.
\STATE Select $\x_t \leftarrow \argmax_{\x \in \cD_t} \langle\widetilde\btheta_t, \Phi(\x)\rangle + \beta_t\lVert \Phi(\x) \rVert_{\bV_t^{-1}}$.
\STATE Play arm $\x_t$ and obtain $y_t$.
\STATE Send $(\Phi(\x_t), y_t) \rightarrow $\textsc{Privatizer}.
\STATE \underline{\textsc{Privatizer:}}
\STATE \textbf{Sending parameters}:
\STATE Obtain $\bH_t, \bh_t$ based on Section~\ref{sec:privacy}.
\STATE Send $\widetilde{\bS}_t = \bS_t + \bH_t, \widetilde\bu_t = \bu_t + \bh_t \rightarrow$\textsc{Server}
\STATE \textbf{Updating parameters}:
\STATE Receive $\x_t, y_t \leftarrow$\textsc{Server}.
\STATE Securely update $\bS_{t+1} \leftarrow \bS_t + \Phi(\x)\Phi(\x)^\top$ (Sec.~\ref{sec:privacy}).
\STATE Securely update $\bu_{t+1} \leftarrow \bu_t + y_i\Phi(\x)$ (Section~\ref{sec:privacy}).
\ENDFOR
\end{algorithmic}
\end{algorithm}
The complete algorithm is summarized in Algorithm~\ref{alg:main_algo}, and prooof is presented in the appendix. Note that we describe the algorithm abstractly for any $\epsilon$-uniformly approximating feature $\Phi$ with dimensionality $m$, and Theorem~\ref{thm:beta} (and the regret bound) hold for any such feature approximation that also satisfies $\sup_{\x \in \cD}\lVert \Phi(\x) \rVert \leq 1$. The algorithm is described in two separate entities, the \textsc{Server} and the \textsc{Privatizer}, where the privatizer entity has access to the raw rewards and contexts, and the server only obtains privatized versions of the statistics. We now present specific $\Phi$ such that we obtain an efficient algorithm.
\begin{remark}[Parameter Release]
\label{rem:param_release}
$\widetilde\mu$ can be determined entirely only with the parameters $\widetilde\bS_t, \widetilde\bu_t$ (Equation~\ref{eqn:update}). If the noise variables $\bH_t, \bh_t$ are constructed such that the resulting parameters satisfy privacy constraints (see next section), these parameters are by design differentially private and hence $\widetilde\mu$ can be released without using the sample path $(\x_t, y_t)_{t\leq T}$.
\end{remark}
\subsection{Noisy Quadrature Fourier Features}
Bochners' theorem~\citep{bochner1933monotone} states that there exists an integral form for stationary $k$, where the integrand is a product of identical features of the inputs:
\begin{equation}
    k(\x-\by) = \int_{\Omega}\begin{pmatrix}
    \sin(\bomega^\top\x) \\
    \cos(\bomega^\top\x)
    \end{pmatrix}^\top \begin{pmatrix}
    \sin(\bomega^\top\by) \\
    \cos(\bomega^\top\by)
    \end{pmatrix}p(\bomega).
\end{equation}
When the above integral is approximated by a Monte-Carlo average, we obtain the powerful Random Fourier Features (RFF,~\citep{rahimi2008random}) approximation. Random Fourier features, while approximating a variety of kernels, are not efficient since $\epsilon_{\text{RFF}}  = \cO(m^{-1/2})$, requiring prohibitively many features $m$ for our purpose. We consider Quadrature Fourier Features (QFF,~\citet{dao2017gaussian}), a stronger approximation that is motivated by numerical integration, and allows $\epsilon$ to decay \textit{exponentially} in $m$. To define QFF, we require that the kernel be Fourier decomposable.
\begin{assumption}[Decomposability of $k$]
Let $k$ be a stationary kernel defined on $\bbR^d \times \bbR^d$ and $k(\x, \by) \leq 1\ \forall \ \x, \by \in \bbR^d$ with a Fourier transform that decomposes product-wise, i.e., $p(\bomega) = \pi_{j=1}^d p_j(\bomega_j)$\footnote{This is satisfied for commonly-used kernels, e.g., squared exponential. Mat\'ern kernels are decomposable when $d=1$. For $d>1$,~\citet{mutny2018efficient} present a modified Mat\'ern kernel that can be used a surrogate.}.
\label{ass:decomposability}
\end{assumption}
\begin{definition}[Quadrature Fourier Features]
Let $\cD = [0, 1]^d$, and $\x, \by \in \cD$. Fix $m = (\bar{m})^d$ for some $\bar{m} > 1$, and let $p(\bomega) = \exp\left(-\sum_{i=1}^d \frac{\omega_i^2\nu_i^2}{2}\right)$ be the Fourier transform of $k$. The QFF features $\Phi(\x)$ is defined as:
\begin{align*}
    \Phi(\x)_i = \begin{cases}
    \sqrt{\Pi_{j=1}^d \sfrac{1}{\nu_i}Q(\omega_{i,j})\cos(\bomega_i^\top\x) }& \text{if } i\leq \frac{m}{2}\\
    \sqrt{\Pi_{j=1}^d \sfrac{1}{\nu_i}Q(\omega_{m-i,j})\sin(\bomega_{m-i}^\top\x) }              & \text{o.w.}
\end{cases}
\end{align*}
$\Phi$ is hence of dimensionality $2m$, and $Q(\omega_{i, j}) = \frac{2^{m-1/2} m!\sqrt{\pi}}{\nu_j m^2 H_{m-1}(\omega_{i, j})}$ and $H_t$ is the $t^{th}$ Hermite polynomial. The set $(\bomega_i)_{i=1}^m$ is the Cartesian product of $\{\bar{\omega}_j\}_{j=1}^{\bar{m}}$, where each element $\bar{\omega}_i \in \bbR$ and is a zero of the $i^{th}$ Hermite polynomial. See~\citet{hildebrand1987introduction} for details.
\end{definition}
\begin{theorem}[QFF Error~\citep{mutny2018efficient}]
\label{thm:qff_approx_error}
Let $\Phi(\cdot), m$ and $\bar{m}$ be as defined above, $\cD = [0,1]^d$ and $\nu=\min_i \nu_i$. Then,
\begin{equation*}
\resizebox{\linewidth}{!}{$
    \underset{\x, \by \in \cD}{\sup} |k(\x, \by) - \Phi(\x)^\top\Phi(\by)| \leq d2^{d-1}\sqrt{\frac{\pi}{2}}\frac{1}{\bar{m}^{\bar{m}}}\left(\frac{e}{4\nu^2}\right)^{\bar{m}}.$
    }
\end{equation*}
\end{theorem}
\begin{remark}
Theorem~\ref{thm:qff_approx_error} implies that the error $\epsilon$ decays exponentially in $m$ when $m > \nu^{-2}$. \citet{mutny2018efficient} evaluate this phase transition in detail, where a break is observed in simulations. For any known kernel $k$ however, we can simply select $m > \nu^{-2}$ to ensure decay. Moreover, for additive kernels, it can be demonstrated that the dependence is exponential in the effective dimension, which can be much less than $d$.
\end{remark}

By adding appropriate $(\bH_t, \bh_t)$ to maintain privacy, we obtain \textit{noisy} quadrature Fourier features (NQFF).
\begin{definition}[Noisy Quadrature Fourier Features (NQFF)]
Let $\Phi : \bbR^{d} \rightarrow \bbR^{m}$ be an $\epsilon-$approximation QFF to the stationary kernel $k$, and $(\bH_t, \bh_t)_{t=1}^T$ be a sequence of perturbations. Then, at any instant $t$, we can define the noisy QFF as $\widetilde{\Phi}(\bX_t) = \begin{bmatrix}
    \Phi(\bX_t) & {\bm 0} \\
    {\bm 0} & \bGamma_t 
\end{bmatrix}$, where $\bGamma_t^\top \bGamma_t = \bH_t$ (i.e., eigendecomposition of $\bH_t$).
\end{definition}
\subsection{Regret Analysis}
We first present the regret bound for \texttt{GP-UCB} with generic $\epsilon$-uniformly approximating features $\Phi$ with dimensionality $m$. Note that this bound is applicable to any approximation technique that satisfies $\sup_{\x \in \cD}\lVert \Phi(\x) \rVert \leq 1$, and suitable $\lambda_{\min}, \lambda_{\max}$ and $\kappa$.
\begin{theorem}[Regret Bound]
\label{thm:main_regret_bound}
Let $k$ be a stationary kernel with the associated RKHS $\cH_k$, and $\cF_m$ be an RKHS with feature $\Phi(\cdot)$ of dimensionality $m$, that $\epsilon$-uniformly approximates every $f \in \cH_k$ when $\lVert f \rVert \leq B$. Furthermore, assume $\lambda_{\min}, \lambda_{\max}$ and $\kappa$ such that they are $(\zeta/2T)$
-accurate and all regularizers $\bH_t \succcurlyeq 0 \ \forall t \in [T]$ are PSD. Then for $(\beta_t)_{t=1}^T$ chosen by Theorem~\ref{thm:beta}, \texttt{GP-UCB} with noisy proximal features obtains the following cumulative regret with probability at least $1-\zeta$:
\begin{align*}
    \fR(T) \leq 2\sqrt{T\beta_T\gamma_T} + \frac{2T^3\sqrt{\beta_T\epsilon}}{3\rho} + 2TB\epsilon. 
\end{align*}
Where $\gamma_T$ is the maximum information gain (Defn.~\ref{def:mig}).
\end{theorem}
\begin{proof}[Proof (Sketch)]
The first key observation is to bound the per-round regret from $f$ with the per-round regret from optimizing $\widehat\mu$. Next, we utilize standard techniques from the analysis of \texttt{GP-UCB} to bound the regret in terms of $\beta_t$ and $\widetilde\sigma_t$ (using Theorem~\ref{thm:beta} twice), and finally provide a bound on $\widetilde\sigma_t$ in terms of the true information gain $\gamma_T$. Summing over all rounds and manipulating proves the result.
\end{proof}
By replacing $\beta_T$ in the result, and manipulating terms, we can conclude that if we have $\Phi$ such that $\epsilon = \cO(\exp(-m))$ and $m = \cO(\text{polylog}(T))$, then we can obtain sublinear regret. Using the properties of QFF from earlier, we can obtain a specific bound as follows.
\begin{corollary}
\label{cor:regret_bound}
Fix $m = 2(6\log T)^d$ and let $k$ be any kernel that obeys Assumption~\ref{ass:decomposability}. Algorithm~\ref{alg:main_algo} run with $m$-dimensional NQFF and noise $\bH_t, \bh_t$ that are $\zeta/2T$-accurate with constants $\lambda_{\max}, \lambda_{\min}$ and $\kappa$ obtains with probability at least $1-\zeta$, cumulative pseudoregret:
\begin{equation*}
\resizebox{\linewidth}{!}{$
    \fR(T) = \cO\left(\sqrt{T\gamma_T}\left(\frac{B\sqrt{\lambda_{\max}}}{\rho} + \sqrt{\log\frac{1}{\zeta} + (\log T^6)^{d+1}} + \frac{\kappa}{\rho}\right)\right).
    $}
\end{equation*}
\end{corollary}
\begin{proof}[Proof (Sketch)]
By Theorem~\ref{thm:qff_approx_error}, we can coarsely bound the approximation error by setting $\bar{m} = \log T^6$ to obtain $\epsilon = \cO(T^{-6})$. Substituting this in Theorem~\ref{thm:main_regret_bound} provides us with the final result.
\end{proof}
\begin{remark}[Selection of $m$]
Note that the analysis presents a bound in terms of the information gain of the true kernel $k$, and hence requires $m = 2(\log T^6)^d$ features. However, an alternate technique will be to bound the information gain of $\tilde k$, which can subsequently be bound with a term of $\cO(\sqrt{m\log T})$. In this case, setting $m = 2(\log T^3)^d$ suffices for no-regret learning, however the obtained regret is (coarsely) $\cO(\sqrt{T}(\log T)^{d+1})$, which can be loose if $\gamma_T = o((\log T)^{d+1})$ (e.g., when $k$ is low-rank).
\end{remark}
\begin{remark}[Feasibility of Kernel Approximations]
\label{rem:m_rff}
The current framework requires $\epsilon=o(T^{-4})$ with $m = \cO(\text{polylog}(T))$ to obtain a no-regret algorithm. Random Fourier Features, while capable of approximating a variety of stationary kernels, decay with $\epsilon = \cO(m^{-1/2})$ which makes them infeasible. For finite-dimensional $\cH_k$, the results manifestly hold with $\epsilon = 0$. 
\end{remark}
\begin{remark}[Unknown $T$]
When $T$ is unknown, we can use a doubling scheme to calculate $m$ and $\epsilon$. To calculate $\epsilon$, we assume $T=1$ for the first round, then assume $T=2$ for the next, and then assume $T=4$ for the next 2 rounds, $T=8$ for the next 4 rounds and so on, and set $\epsilon = \cO(t^{-5})$, for instance, within each ``period'' of length $t$ between doubling of $T$ to calculate $m$. We see that the regret is at most $ \widetilde\cO(\sqrt{t})$ for this period. Since there are at most $\cO(\log T)$ such periods, and $t \leq T$, the total regret is $\cO((\log T)\sqrt{T})$.
\end{remark}
\section{\texttt{GP-UCB} with Differential Privacy}
\label{sec:privacy}
We now present the mechanism to ensure Algorithm~\ref{alg:main_algo} is differentially private. Proceeding with the standard definition of differential privacy (Equation~\ref{eqn:dp}) for the streaming setting, however, is infeasible (i.e., leading to linear regret, see Claim 13 of~\citet{shariff2018differentially}). We therefore work with a modified notion of privacy that is the standard for sequential decision-making~\citep{shariff2018differentially, vietriprivate}.
\begin{definition}[Joint Differential Privacy (JDP)]
\label{def:jdp}
Let $S = (\cD_i, y_i)_{i=1}^T$ and $S' = (\cD'_i, y'_i)_{i=1}^T$ be two sequences such that $(\cD_i, y_i) = (\cD'_i, y'_i)$ for all $i \neq t$, and $\cS_{-t} \subseteq \cD_1 \times ... \times \cD_{t-1} \times \cD_{t+1} \times ... \times \cD_{T}$ denote a sequence of actions except the $t^{th}$. An algorithm $\cA$ is $(\alpha, \beta)$-JDP under continual observation if for any $t \in [T], S, S'$, it holds that $\bbP(\cA(S) \in \cS_{-t}) \leq e^{\alpha}\bbP(\cA(S') \in \cS_{-t}) + \beta$.
\end{definition}
The only change in the JDP setting (compared to standard DP) is that the algorithm is allowed to be non-private at time $t$ with respect to $\cD_t$ (i.e., the active decision set). This is crucial as standard DP would imply that for any two actions $\x, \x' \in \cD_t, \bbP(a_t = \x ) \approx \bbP(a_t = \x')$ and the algorithm would incur linear regret.

\begin{algorithm}[t]
\caption{\textsc{Privatizer} under JDP}
\small
\label{alg:jdp_algo}
\begin{algorithmic}
\STATE \textbf{Initialize}: Binary tree $\cT$.
\FOR{round $t=1, 2, ..., T$}
\STATE \textbf{Sending parameters}:
\STATE Obtain $\widetilde{\bS}_t, \widetilde\bu_t$ by traversing $\cT$ to node $t$.
\STATE Send $\widetilde{\bS}_t, \widetilde\bu_t \rightarrow$\textsc{Server}. 
\STATE \textbf{Updating parameters}:
\STATE Receive $\x_t, y_t \leftarrow$\textsc{Server}.
\STATE Insert $[\Phi(\x_t), y_t]^\top[\Phi(\x_t), y_t]$ into $\cT$.
\STATE Update noise values $\bn$ on the inserted path $\cT$.
\ENDFOR
\end{algorithmic}
\end{algorithm}
\subsection{Approximate \texttt{GP-UCB} with JDP}
\label{sec:privacy_jdp}
Our approach involves perturbing $(\bS_t, \bu_t)$ by noise $(\bH_t, \bh_t)$ to ensure JDP, and it is summarized in Algorithm~\ref{alg:jdp_algo}. Observe that the estimates $(\widetilde\bS_t, \widetilde\bu_t)$ are noisy cumulative sums of $\bS_t = \sum_{\tau=1}^{t-1} \Phi(\x_\tau)\Phi(\x_\tau)^\top, \bu_t = \sum_{\tau=1}^{t-1} y_\tau\cdot\Phi(\x_\tau)$. This additive structure naturally suggests that we utilize a matrix variant of the tree-based mechanism~\citep{dwork2010differential, shariff2018differentially} to maintain $(\widetilde\bS_t, \widetilde\bu_t)$. We consider the matrix $\bN_t = \left[\Phi(\bX_t), \ \by_t\right]^\top\left[\Phi(\bX_t), \ \by_t\right] \in \bbR^{m+1 \times m+1}$ and compute this matrix via the tree-based mechanism. The advantage of maintaining $\bN_t$ is that $\bN_{t+1} = \bN_{t} + [\Phi(\x_t),\ y_t]^\top [\Phi(\x_t),\ y_t]$ and the top $m \times m$ submatrix of $\bN_{t}$ is $\bS_t$ and the first $m$ entries of the last column of $\bN_t$ is $\bu_t$, giving us the required estimates.

\textbf{Tree-Based Mechanism}. The tree-based mechanism~\citep{dwork2010differential} estimates the rolling sum of any series $\bn_1, \bn_2, ... $ via a binary tree. Let $P_{m+1}$ be a probability distribution over $\bbR^{m+1 \times m+1}$. A trusted entity (in our case, the \textsc{Privatizer}), maintains a binary tree $\cT$ whose $t^{th}$ leaf node stores $\bn_t = [\Phi(\x_t) \, y_t]^\top[\Phi(\x_t) \, y_t] + (\sfrac{1}{\sqrt{2}})\left(\bupsilon_t\top+\bupsilon_t\right)$, where $\bupsilon_t$ is a sample from $P_{m+1}$. Each parent node stores 
the sum of its children. Now, to compute $(\widetilde\bS_t, \widetilde\bu_t)$ we traverse $\cT$ to the $t^{th}$ leaf node, and sum the values at each node. Since the path length traversed is $1+\ceil{\log_2 T}$, we can rewrite $\widetilde\bS_t = \bS_t + \bH_t$ where $\bH_t$ is the sum of at most $n = 1+\ceil{\log_2 T}$ samples from $P_{m+1}$. We now describe selecting $P_{m+1}$ to provide a JDP guarantee.
\begin{lemma}[JDP]
\label{lem:jdp_noise}
Let $P_{m+1}$ be a composition of $(m+1)^2$ zero-mean normal variables with variance $\sigma^2_{\alpha, \beta}$. If $\sigma_{\alpha, \beta} > 16n(1 + B^2 +2\rho^2\log(8T/\beta))\ln(10/\beta)^2/\alpha^2$, then Alg.~\ref{alg:main_algo} with \textsc{Privatizer} following Alg.~\ref{alg:jdp_algo} is $(\alpha, \beta)-$jointly differentially private.
\end{lemma}
\begin{proof}[Proof (Sketch)]
First note that since $y_t$ is sub-Gaussian with mean at most $B$ (since $\lVert f \rVert_k \leq B$), we can apply a standard Chernoff bound to ensure that with probability at least $1-\beta/4$, for each $(y_\tau)_{\tau \in [T]}$ simultaneously, $|y_t|^2 \leq B^2 + 2\rho^2\log(8T/\beta)$. Using this bound we can ensure that each datum has a bounded $L_2$-norm of $1 + B^2 + 2\rho^2\log(4T/\beta)$ (since $\lVert \Phi(\x) \rVert_2 \leq 1$). Based on the composition for zero-concentrated DP~\citep{bun2016concentrated}, we see that for $(\alpha, \beta)$-JDP, we require that each of the at most $n = 1+\ceil{\log_2 T}$ nodes maintains $(\alpha/\sqrt{8n\ln(2/\beta)}, \beta/2)$-DP. With the $L_2$ sensitivity result from earlier, we see that $\sigma^2_{\alpha, \beta} = 16n(1 + B^2 +2\rho^2\log(8T/\beta))\ln(10/\beta)^2/\alpha^2$ provides $(\alpha, \beta)$-JDP, finishing the proof.
\end{proof}
Recall that our regret bound (Corollary~\ref{cor:regret_bound}) scales with the parameters $\lambda_{\min}, \lambda_{\max}$ and $\kappa$. It remains to provide these quantities under the selected $P_{m+1}$ such that they are accurate (Defn.~\ref{def:accurate}), and provide final regret bounds based on the properties of $P_{m+1}$. As remarked in~\citet{shariff2018differentially}, we must shift the noise matrix to ensure that all noise samples $\bH_t$ are PSD.
\begin{lemma}[Accurate Spectrum under JDP]
\label{lem:jdp_spectrum}
For any $\zeta > 0$, when $P_{m+1}$ is selected according to Lemma~\ref{lem:jdp_noise} and $\bH_t, \bh_t$ are constructed according to Alg.~\ref{alg:jdp_algo}, the following $\lambda_{\min}, \lambda_{\max}$ and $\kappa$ are $(\zeta/2T)-$accurate:
\begin{equation*}
\resizebox{\linewidth}{!}{$
    \lambda_{\min} = \Lambda, \lambda_{\max} = 3\Lambda, \kappa = \sigma_{\alpha, \beta}\sqrt{\frac{n}{\Lambda}}\left(\sqrt{m} + \sqrt{2\ln\frac{2T}{\zeta}}\right).$
    }
\end{equation*}
Here $\Lambda = \sigma_{\alpha, \beta}\sqrt{2n}(4\sqrt{m} + 2\ln(2T/\zeta))$.
\end{lemma}
\begin{proof}
This proof is identical to Proposition 11 from \citet{shariff2018differentially} with our noise model.
\end{proof}
\begin{corollary}[$(\alpha, \beta)$-JDP Regret Bound]
\label{cor:regret_bound_jdp}
Fix $m = 2(6\log T)^d$ and let $k$ be any kernel that obeys Assumption~\ref{ass:decomposability}. Algorithm~\ref{alg:main_algo} run with $m$-dimensional NQFF and noise such that it maintains $(\alpha, \beta)-$JDP obtains with probability at least $1-\zeta$, cumulative pseudoregret:
\begin{equation*}
\resizebox{\linewidth}{!}{$
    \fR(T) = \cO\left(\sqrt{T\gamma_T}\left((\ln T)^{\frac{d+2}{4}}(\frac{1}{\alpha}\log\frac{1}{\beta}\log\frac{1}{\zeta})^{\frac{1}{2}} + (\ln T)^{\frac{d+1}{2}}\right)\right).
    $}
\end{equation*}
\end{corollary}
The proof for Corollary~\ref{cor:regret_bound_jdp} follows directly by substituting the results from Lemma~\ref{lem:jdp_spectrum} into Corollary~\ref{cor:regret_bound}.
\begin{remark}[Dependence on $m$]
Since the factors $\lambda_{\min}, \lambda_{\max}$ and $\kappa$ admit a dependence of $\cO(\sqrt{m})$ on the dimensionality of $\Phi$, we require $m = o(\sqrt{T})$ features to guarantee no-regret learning under our approach. This constraint is complementary to the constraint on $m$ from kernel approximation (Remark~\ref{rem:m_rff}), and mandates that even when the approximation $\tilde k$ has small $\gamma_T$ (i.e., $\gamma_T = o(\text{polylog}(T))$, we require small $m$.
\end{remark}
\subsection{Approximate \texttt{GP-UCB} with Local JDP}
\label{sec:privacy_ldp}
In many settings, the existence of a trusted entity (e.g., \textsc{Privatizer}) is not possible. For instance, consider the task of a centralized server learning a bandit algorithm in the case when each user $t$ does not wish $(\cD_t, y_t)$ to be sent to the server at all (even to select $\x_t$). We can select $\x_t$, however, by sending the algorithm's (privatized) parameters to each user individually and collecting updated parameters after $\x_t$ has been played by the user $t$. Here, we employ an alternative definition of privacy known as local JDP.
\begin{definition}[Locally Joint Differential Privacy (Local JDP)]
\label{def:ldp}
A mechanism $g : \cX \rightarrow \cZ$ is $(\alpha, \beta)$-locally differentially private~\citep{bebensee2019local} (LDP) if for any $\x, \x' \in \cX, \bbP(g(\x) \in \cZ) \leq e^\alpha \bbP(g(\x') \in \cZ) + \beta$.
For any sequence $(\cD_t, y_t)_{t=1}^T$, an algorithm $\cA$ protects locally joint differentially privacy (Local JDP) if for any $t$, $\cA$ is locally differentially private with respect to each $(\cD_\tau, y_\tau)$ simultaneously where $\tau \neq t$.
\end{definition}
This definition combines joint differential privacy (operating globally) with local differential privacy (operating individually). It is important to note that local JDP is weaker than LDP~\citep{bebensee2019local}, since LDP would require local privacy with respect to $(\cD_t, y_t)$ as well. It is a stronger privacy guarantee than JDP, since it requires $\cA$ to be private to each user simultaneously.
\begin{lemma}[Local JDP implies JDP]
\label{lem:local_implies_jdp}
Any $(\alpha, \beta)-$local JDP algorithm $\cA$ protects $(\alpha,\beta)$-JDP for each $t \in [T]$.
\end{lemma}
\begin{proof}[Proof (Sketch)]
For any $t \in [T]$, any two $t$-neighboring sequences $S$ and $S'$ only differ in the $t^{th}$ entries $(\cD_t, y_t)$ and $(\cD'_t, y'_t)$. Since $\cA$ is $(\alpha, \beta)-$locally JDP, $\bbP(a_{t'}(\cD_t, y_t)) \in \cS_{t'}) \leq e^\alpha \bbP(a_{t'}(\cD'_t, y'_t)) \in \cS_{t'}) + \beta$ for all $t' \neq t$, from which the result follows.
\end{proof}
Since a trusted entity does not exist, learning is done by sending the parameters directly to the users (ref. clients). We outline a server-client protocol and associated algorithm for $(\alpha, \beta)$-local JDP Gaussian Process bandit optimization in Algorithm~\ref{alg:ldp_algo}. This algorithm requires noise added individually to $(\Phi(\x_t), y_t)$ (instead of $(\bS_t, \bu_t)$). We achieve this by perturbing $\bS_t$ and $\bu_t$ separately with $\bN_t \in \bbR^{m \times m}$ where $\bN_t(i, j) \sim \cN(0, \sigma^2_X)$ for $i \geq j$ and $\bN_t(i, j) = \bN_t(j, i)$ otherwise and $\bn_t \in \bbR^m$ is such that $\bn_t(i) \sim \cN(0, \sigma^2_u)$. The variances $\sigma^2_X$ and $\sigma^2_u$ are chosen to ensure $(\alpha/2, \beta/2)$ respectively, securing $(\alpha, \beta)$-Local JDP.
\begin{lemma}[Noise for Local JDP]
\label{lem:ldp_noise}
Algorithm~\ref{alg:ldp_algo} is $(\alpha, \beta)-$locally JDP whenever,
\begin{align*}
    \sigma^2_X \geq \frac{8}{\alpha^2}\ln\frac{5}{2\beta},\ 
    \sigma^2_u \geq \frac{8}{\alpha^2}\left(B^2 + 2\ln\frac{8m}{\delta}\right)\ln\frac{5}{\beta}.
\end{align*}
\end{lemma}
\begin{proof}
We first note that the $L_2-$sensitivity of each element within $\Phi(\x_t)^\top\Phi(\x_t)$ is 1 by the fact that $\lVert \Phi(\x) \rVert \leq 1$. Next, note that the $L_2-$ sensitivity of each element of $y_t\Phi(\x_t)$ is with probability at least $1-\beta/4$ at most $B + \rho\sqrt{2\log\frac{8m}{\delta}}$ ($y_t$ is Gaussian with mean at most $B$). Now, by the Gaussian mechanism for local DP~\citep{dwork2014algorithmic}, we have that for $\sigma^2_x \geq \frac{8\ln(2.5/\beta)}{\alpha^2}$ and $\sigma_u \geq \frac{8(B^2 + 2\log\frac{8m}{\delta})\ln(5/\beta)}{\alpha}$, both $ \Phi(\x_t)^\top\Phi(\x_t) + \bN_t$ and $y_t\Phi(\x_t) + \bn_t$ are $(\sfrac{\alpha}{2}, \sfrac{\beta}{2})-$locally DP.
\end{proof}
\begin{algorithm}[t]
\caption{\texttt{GP-UCB} with Local JDP}
\small
\label{alg:ldp_algo}
\textsc{Server:}
\begin{algorithmic} 
\STATE \textbf{Initialize}: $\bS_1 = {\bf 0}, \bu_1 = {\bm 0}$.
\FOR{round $t=1, 2, ..., T$}
\STATE Send $\widetilde\bS_t, \widetilde\bu_t \rightarrow$\textsc{Client}$(t)$.
\STATE Receive updated $\widetilde\bS_{t+1}, \widetilde\bu_{t+1} \leftarrow$\textsc{Client}$(t)$.
\ENDFOR
\end{algorithmic}
\textsc{Client$(t)$:}
\begin{algorithmic}
\STATE Initialize $\sigma^2_X$ and $\sigma^2_u$ according to Lemma~\ref{lem:ldp_noise}.
\STATE Receive $\cD_t$ from environment.
\STATE Receive $\widetilde{\bS}_t, \widetilde\bu_t \leftarrow$\textsc{Server}.
\STATE Set $\bV_t \leftarrow \widetilde{\bS}_t + \lambda\bI, \widetilde\btheta_t \leftarrow \bV_t^{-1}\widetilde\bu_t$.
\STATE Compute $\beta_t$ based on Theorem~\ref{thm:beta}.
\STATE Select $\x_t \leftarrow \argmax_{\x \in \cD_t} \langle\widetilde\btheta_t, \Phi(\x)\rangle + \beta_t\lVert \Phi(\x) \rVert_{\bV_t^{-1}}$.
\STATE Play arm $\x_t$ and obtain $y_t$.
\STATE Sample $\bN_t, \bn_t$ using $\sigma^2_X, \sigma^2_u$.
\STATE Send $\widetilde\bS_{t+1} \rightarrow \bS_t+\Phi(\x_t)\Phi(\x_t)^\top + \bN_t \rightarrow $\textsc{Server}.
\STATE Send $\widetilde\bu_{t+1} \rightarrow \bu_t+y_t\Phi(\x_t)+ \bn_t \rightarrow $\textsc{Server}.
\end{algorithmic}
\end{algorithm}
It remains to bound the spectral parameters ($\lambda_{\min}, \lambda_{\max}$ and $\kappa$) in order to obtain regret bounds.
\begin{lemma}(Spectrum for Local JDP)
\label{lem:ldp_spectrum}
For any $\zeta > 0$, fix $\Lambda = \sqrt{T}(4\sqrt{m} + 2\ln(2T/\zeta))$. When $P_{m+1}$ is selected according to Lemma~\ref{lem:ldp_noise} and $\bH_t, \bh_t$ are constructed according to Alg.~\ref{alg:ldp_algo}, the following are $(\zeta/2T)-$accurate:
\begin{align*}
    \lambda_{\min} = \sigma_x\Lambda, \lambda_{\max} = 3\sigma_x\Lambda \text{ and,} \kappa = \sigma_{u}\sqrt{mT\Lambda^{-1}}.
\end{align*}
\end{lemma}
\begin{proof}[Proof (Sketch)]
The proof is identical to Lemma~\ref{lem:jdp_spectrum} except critically that in this case, $\bH_t$ (resp. $\bh_t$) is the sum of $t$ matrices $\bN_t$ (resp. $\bn_t$), with total variance $t\sigma^2_X$ (resp. $t\sigma^2_u$). Therefore, we can bound $\lVert \bH_t \rVert_2 \leq \sigma_X\sqrt{T}(4\sqrt{m} + 2\ln(2T/\zeta))$ and $\lVert \bh_t \rVert_2 \leq \sigma_u\sqrt{mT}$, which gives the result identical to Lemma~\ref{lem:jdp_spectrum}.
\end{proof}
\begin{corollary}[$(\alpha, \beta)-$Local JDP Regret Bound]
\label{cor:regret_bound_ldp}
Fix $m = (6\log T)^d$ and let $k$ be any kernel that obeys Assumption~\ref{ass:decomposability}. Algorithm~\ref{alg:main_algo} run with NQFF and noise $\bH_t, \bh_t$ that maintains $(\alpha, \beta)-$local JDP obtains with probability at least $1-\zeta$, cumulative pseudoregret:
\begin{equation*}
\fR(T) = \cO\left(T^{\frac{3}{4}}(\ln T)^{\frac{d+2}{4}}\sqrt{\frac{\gamma_T}{\alpha}\ln\frac{1}{\beta}\ln\frac{1}{\zeta}} \right).
\end{equation*}
\end{corollary}
The proof for Corollary~\ref{cor:regret_bound_ldp} follows directly by substituting the results from Lemma~\ref{lem:ldp_spectrum} into Corollary~\ref{cor:regret_bound}.
\begin{remark}[JDP vs. Locally JDP Regret]
\label{rem:suboptimal_ldp_regret}
Our algorithm for the locally JDP setting obtains $\widetilde\cO(T^{\sfrac{3}{4}})$ regret in contrast to the JDP regret, which is close to the minimax optimal rate of $\widetilde\Omega(\sqrt{T})$ for squared-exponential and Mat\'ern kernels~\citep{scarlett2017lower}. It is evident that this suboptimality is introduced by the $\widetilde\cO(T)$ noise added via $\bH_t$. However, we conjecture that in the absence of any known structure between the chosen actions $\x_1, ..., \x_{t-1}$, it is impossible to add correlated noise samples (i.e., such that the overall variance is $o(T)$) while maintaing local DP, as typically the environment selects $\cD_t$ independently of $\cD_{t-1}$.
\end{remark}
\section{Experiments}
We conduct experiments primarily around the noisy Quadrature features for GP optimization, and consider the Joint DP setting. For more experimental results on the approximation guarantees of QFF, please refer to the appendix and experimental section of \citet{mutny2018efficient}, that analyse the efficiency of quadrature features in approximating stationary kernels.

We conduct experiments with input dimensionality $d = 2$ and selecting the squared-exponential kernel with variance $1$, i.e., $k(\x, \by) = \exp(-\lVert \x - \by \rVert_2^2/2)$ for simplicity. This choice was made as we essentially wish to demonstrate that the algorithms are private in practice for toy experiments, as larger dimensionalities $(d > 5)$ are rarely seen in practice~\citep{mutny2018efficient} and would require additive assumptions for efficient inference~\citep{munkhoeva2018quadrature}.

\subsection{Experimental Setup}
We construct $f$ by randomly sampling a set of points $\cI$ from $\cB_d(2)$ such that $|\cI| = 4$ and randomly generate $\balpha$ from the unit $L_1$ ball $\cB_d(1)$ (therefore, we assume $B = 1$). For any input point $\x$, $f(\x)$ can then be denoted as $f(\x) = \sum_{i=1}^{|\cI|} \alpha_i k(\x_i, \x)$, where $\x_1, ..., \x_{|\cI|}$ belong to $\cI$. We consider $\cD$ to be a random sample of size $n$ drawn from $\cB_d(2)$ ($n$ may be variable, but is specified prior to each experiment). We draw $\cD_t$ such that at least 1 sample $\x$ from $\cD_t$ satisfies $f(\x) \geq 0.8$ and the others satisfy $f(\x) \leq 0.6$, ensuring a suboptimality gap of at least $0.2$ (this is implemented somewhat crudely by iterative sampling). At each round $t$, the agent is presented with a random $\cD_t$ and it obtains a reward $y_t$ drawn from the distribution $\text{Ber}(f(\x))$ and hence $|\varepsilon_t | \leq 1$ and $\bbE[y_t] = f(\x)$. Additionally, we see that the variance $\rho^2 = f(\x)(1-f(\x))$ for this case, but that is bounded from above by $1/4$. For simplicity, we restrict ourselves to Bernoulli rewards. This model, while ensuring sub-Gaussianity, also ensures that the rewards are bounded, and hence removes an additional logarithmic factor from the sensitivity analysis for the JDP setting. This can be observed by directly applying $L_2$-sensitivity to the JDP noise (Lemma~\ref{lem:jdp_noise}), and ignoring the probabilistic argument.

\textbf{Effect of $\alpha$}. We first examine the effect of adjusting the privacy level $\alpha$. We fix $n=25$, $\beta = 0.1$ and set $T=1024$ (similar to~\citet{mutny2018efficient}). We run 20 trials and compare the performance at $\alpha = 0.1, 0.5, 1, 10$ (averaged over 20 trials). The regret scales as predicted with decreasing $\alpha$ (Figure~\ref{fig:alpha_regret}).
 \begin{figure}[t!]
 \centering
  \includegraphics[width=0.8\linewidth]{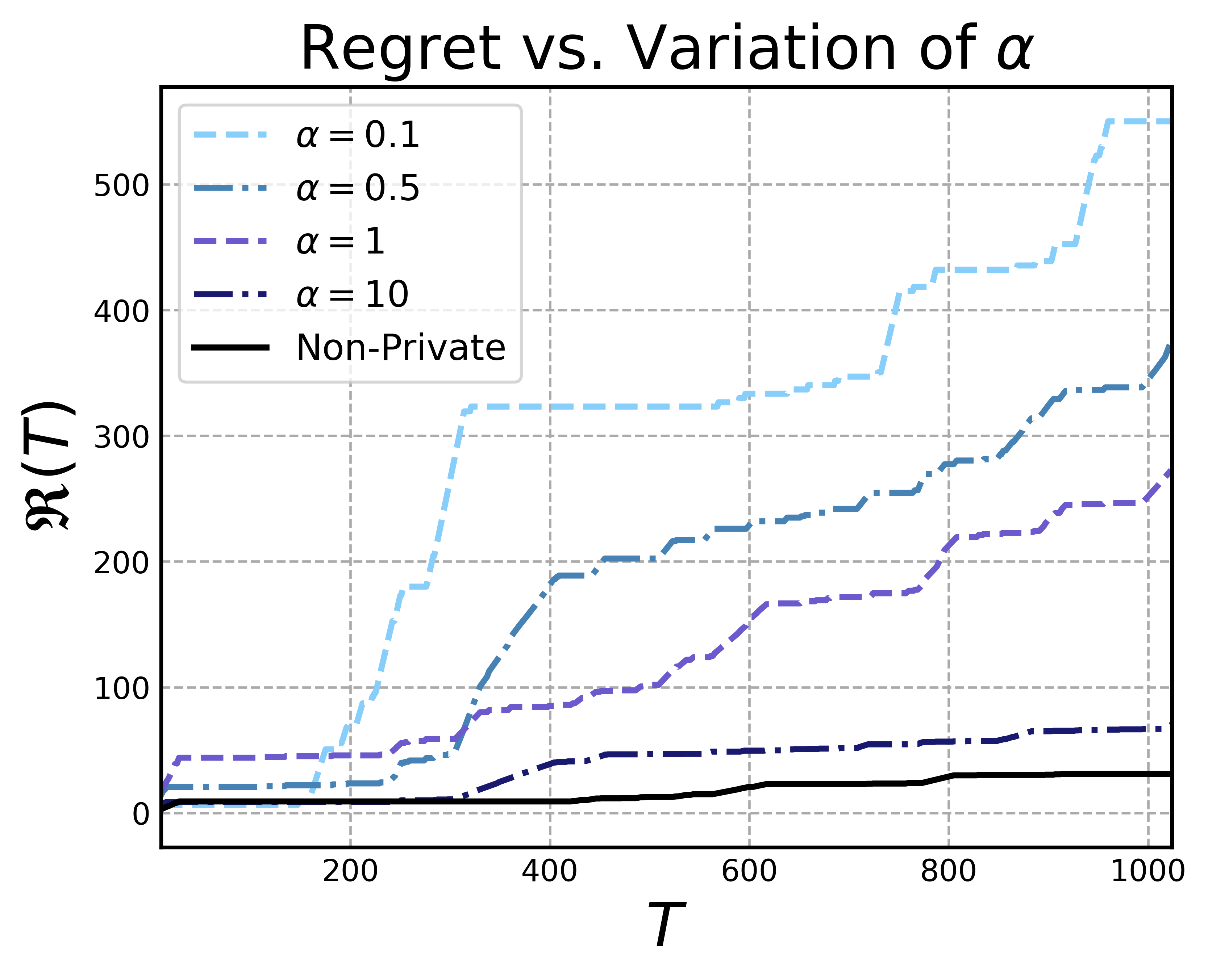}
  \caption{An experimental comparison of approximate \texttt{GP-UCB} for various values of privacy budget $\alpha$.}
  \label{fig:alpha_regret}
\end{figure}

\textbf{Effect of $\beta$}. Next, we examine the effect of adjusting the privacy failure probability $\beta$. We fix $n=25$, $\alpha = 1$ and set $T=1024$ (similar to~\citet{mutny2018efficient}). We run 20 trials and compare the performance at $\beta = 0.01, 0.1, 0.5, 0.99$ (averaged over 20 trials). The regret increases with decreasing $\beta$, summarized in Figure~\ref{fig:beta_regret}.
 \begin{figure}[t!]
 \centering
  \includegraphics[width=0.8\linewidth]{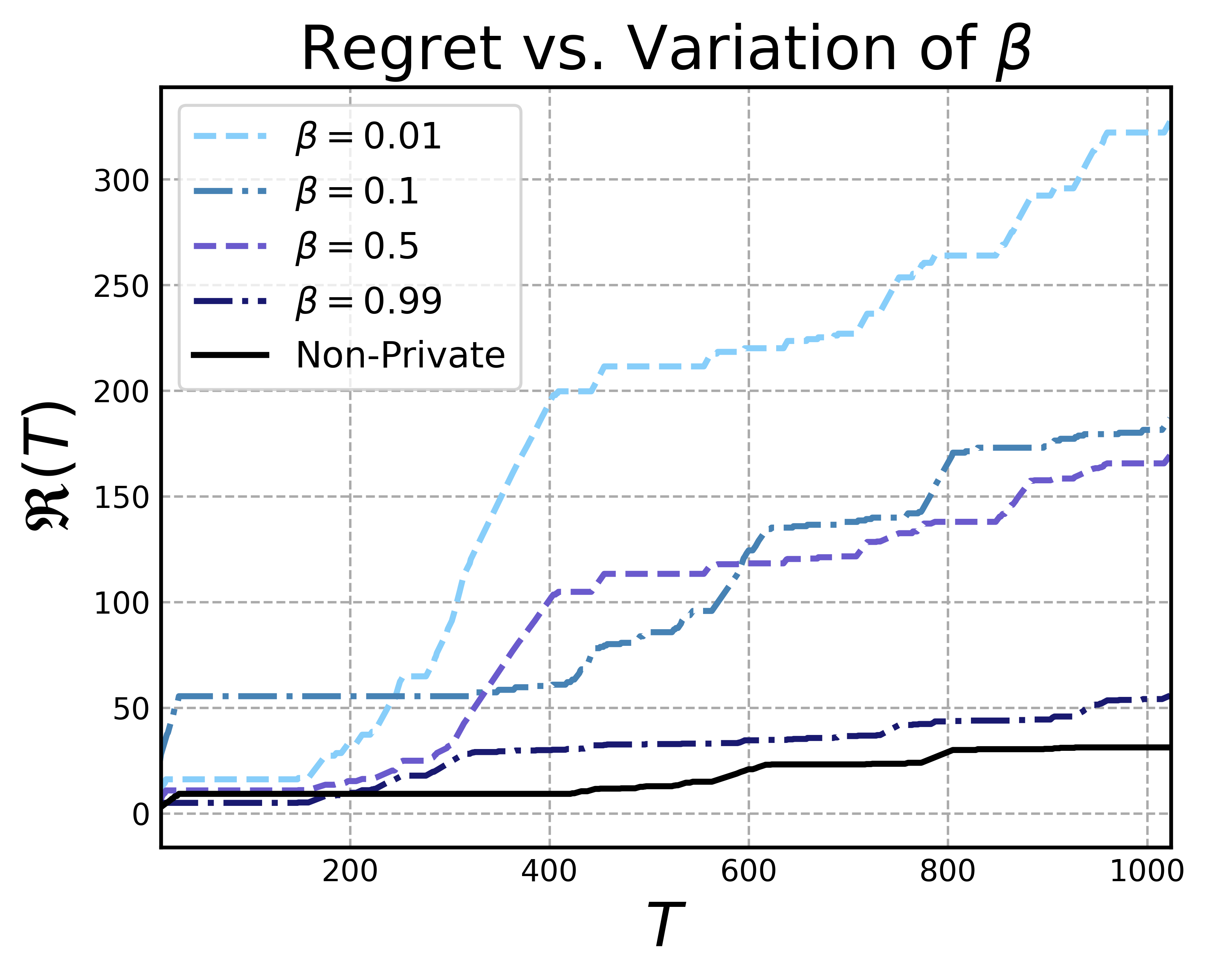}
  \caption{An experimental comparison of approximate \texttt{GP-UCB} for various values of privacy failure probability $\beta$.}
  \label{fig:beta_regret}
\end{figure}
\begin{table}[t]
\centering
\begin{tabular}{c|ccc}
\hline
Alg. & \texttt{camel} & \texttt{styb}  & \texttt{mw}  \\ \hline
Non-Private & 519 & 885 & 901 \\
$\alpha=10$ & 775 & 1667 & 1558 \\
$\alpha=1$ & 1029 & 2680 & 2883 \\
$\alpha=0.1$ & 3324 & 4493 & 5002 \\ \hline
\end{tabular}
\caption{Cumulative regret at $T=10K$ averaged over 10 trials on functions from~\citet{mutny2018efficient}.}
\label{tab:main}
\end{table}
\subsection{Additional Benchmarks}
In addition to the environment proposed earlier, we additionally evaluate the JDP algorithm on previous benchmark environments. We consider the functional environments for the Camelback (\texttt{camel}), Stybtang-20 (\texttt{styb}) and Michalewicz-10 (\texttt{mw}) benchmarks from~\citep{mutny2018efficient}. We observe a consistent increase in regret as the privacy budget $(\alpha)$ is reduced (Table~\ref{tab:main}). While the bound predicts a $\alpha^{-\frac{1}{2}}$ deterioration, we observe a larger effect, which suggests that stronger analyses can close the gap.

\section{Discussion and Concluding Remarks}
In this paper, we presented the first \textit{no-regret} algorithmic framework for differentially-private Gaussian Process bandit optimization for a class of stationary kernels in both the joint DP and local DP settings, extending the literature on private bandit estimation beyond multi-armed~\citep{mishra2015nearly} and linear~\citep{shariff2018differentially} problems. We rigorously analyse the proposed algorithms and demonstrate their provable efficiency in terms of regret, computation and privacy. Our work additionally introduces several new avenues for further research - while the dependence of the achieved pseudoregret on $T$ is near-optimal in the JDP setting, the local JDP setting introduces an additional $\cO(T^{\sfrac{1}{4}})$ which we conjecture is necessary owing to the nested estimation problems involved (Remark~\ref{rem:suboptimal_ldp_regret}). Additionally, developing lower bounds on private GP regret and efficient kernel approximations for non-stationary kernels are valuable pursuits of inquiry.

\section*{Acknowledgements}
We would like to thank Dr. Alex Pentland for his helpful comments, and the anonymous reviewers for their feedback and suggestions. This work was supported by the MIT Trust::Data Consortium.
\bibliographystyle{icml2021}
\bibliography{refs}
\ifx\doappendix\undefined
undefed
\else
  \if\doappendix1
\newpage
\onecolumn
\appendix
\section{Appendix}
\subsection{Preliminary Results}
\begin{lemma}[Chernoff with Maximum Mean Bound]
\label{lem:subgaussian_max_mean}
Let $X$ be any $\sigma$-sub-Gaussian random variable with mean $\mu \leq \mu^*$ for some constant $\mu^*$. Then, with probability at least $1-\delta$,
\begin{align*}
    |X| \leq \mu^* + \sigma\sqrt{2\ln\left(\frac{2}{\delta}\right)}.
\end{align*}
\end{lemma}
\begin{proof}
$X$ is sub-Gaussian with variance, therefore by a Chernoff bound,
\begin{align}
    \bbP\left(X-\mu > t\right) &\leq \exp\left(-\frac{t^2}{2\sigma^2}\right) \\
    \implies \bbP\left(X> t+\mu\right) &\leq \exp\left(-\frac{t^2}{2\sigma^2}\right) \\ \intertext{Subsitituing $t' = t+\mu$,}
    \implies \bbP\left(X> t'\right) &\leq \exp\left(-\frac{(t'-\mu)^2}{2\sigma^2}\right) \\
    \implies X &\leq \mu + \sigma\sqrt{2\ln\left(\frac{1}{\delta}\right)} \tag{With probability at least $1-\delta$}.
\end{align}
The same can be derived for the other tail. By combining both statements with a union bound we get the result.
\end{proof}

\begin{lemma}[DP with probabilistic $L_2$ sensitivity]
\label{lem:prob_gaussian_mechanism}
Let $f : \bbR^{|\cX|} \rightarrow \bbR^m$ be an arbitrary $d$-dimensional real-valued function with $L_2-$sensitivity $\Delta$ with probability at least $1-\delta'$, and $\varepsilon \in (0, 1)$ be arbitrary. For $c^2 > 2\ln(1.25/\delta)$, the Gaussian Mechanism with parameter $\sigma \geq c\Delta/\varepsilon$ is $(\varepsilon, \delta+\delta')-$differentially private.
\end{lemma}
\begin{proof}
Denote two adjacent samples in $\cX$ as $\x, \x'$. We release $\by = f(\x) + \eta$ and $\by' = f(\x') + \eta$, where $\eta$ is sampled from the corresponding Gaussian. For any arbitrary subset $S$ of $\bbR^m$, 
\begin{align}
    \bbP(\by \in S) &= \bbP(\by \in S ; \text{ sensitivity of $f$ is $\Delta$}) + \bbP(\by \in S ; \text{sensitivity of $f$ is not $\Delta$}) \\
    &= \bbP(\by \in S | \text{ sensitivity of $f$ is $\Delta$})\bbP(\text{sensitivity of $f$ is $\Delta$}) + \delta'\bbP(\by \in S | \text{ sensitivity of $f$ is not $\Delta$}) \\
    &\leq \bbP(\by \in S | \text{ sensitivity of $f$ is $\Delta$})\bbP(\text{sensitivity of $f$ is $\Delta$}) + \delta' \\
    &\leq e^{\varepsilon}\left[\bbP(\by' \in S | \text{ sensitivity of $f$ is $\Delta$}) +\delta\right]\bbP(\text{sensitivity of $f$ is $\Delta$})+\delta' \\
    &= e^{\varepsilon}\bbP(\by' \in S | \text{ sensitivity of $f$ is $\Delta$})\bbP(\text{sensitivity of $f$ is $\Delta$}) +\delta\bbP(\text{sensitivity of $f$ is $\Delta$})+\delta' \\
    &\leq e^{\varepsilon}\bbP(\by' \in S ; \text{ sensitivity of $f$ is $\Delta$}) +\delta+\delta' \\
    &\leq e^{\varepsilon}\bbP(\by' \in S) +\delta+\delta'.
\end{align}
The second inequality is obtained by the Gaussian Mechanism (Theorem A.1 of Dwork and Roth~\citep{dwork2014algorithmic}).
\end{proof}
\begin{lemma}[Existence of Proximal Space~(Lemma 4 of \citep{mutny2018efficient})]
\label{lem:existence_proximal_space_app}
Let $k$ be a kernel defining $\cH_k$ and $f \in \cH_k$, its RKHS, such that the spectral characteristic function is bounded by $B$. Assuming that the defining points of $f$ come from the set $\cD$, let $\cF_m$ be an approximating space with a mapping $\Phi$ such that this mapping is an $\epsilon$-approximation to the kernel $k$. Then there exists $\widehat\mu \in \cF_m$ (with corresponding feature $\widehat\btheta$ such that $\widehat\mu(\x) = \langle \widehat\btheta, \Phi(\x) \rangle$), such that $\sup_{\x \in \cD} \left| \widehat\mu(\x) - f(\x) \right| \leq B\epsilon$. 

Assuming the spectral characteristic function for $f$ is given by $\balpha(\omega) = \sum_{j \in \cI} \alpha_j\exp(i\omega^\top \x_j)$, then $\widehat\mu(\x) = \sum_{j \in \cI} \alpha_j\Phi(\x)^\top\Phi(\x_j)$ and the corresponding $\widehat\btheta = \sum_{j \in \cI}\alpha_j\phi(\x_j)$ for the index set $\cI$ defining $f$.
\end{lemma}

\begin{lemma}[Norm Bound for Proximal Function]
Let $\widehat\mu \in \cF_m$ denote the $\epsilon$-approximation of $f$ given by Lemma~\ref{lem:existence_proximal_space_app} and $\widehat{\btheta}$ denote the corresponding feature representation. Then $\lVert \widehat\btheta \rVert_2 \leq B$.
\label{lem:norm_bound_proximal_function}
\end{lemma}
\begin{proof}
Recall that by the Representer Theorem, $\widehat\btheta = \sum_{i \in \cI}\alpha_i\Phi(\x_i)$ for some (possibly infinite) index set $\cI \subseteq \cD$. Then, we can write $\lVert \widehat\btheta \rVert_2^2 = \left\langle \widehat\btheta, \widehat\btheta \right\rangle_{\cF_m} = \sum_{i, j \in \cI^2} \alpha_i\alpha_j\left(\Phi(\x_i)^\top\Phi(\x_j)\right)$. Then, we can utilize the property that $\widehat\btheta$ is an $\epsilon$-approximation of $\mu_t \in \cH_k$:
\begin{align}
    \lVert \widehat\btheta \rVert_2^2 &= \sum_{i, j \in \cI^2} \alpha_i\alpha_j\left(\Phi(\x_i)^\top\Phi(\x_j)\right) \\
    &\leq \sum_{i, j \in \cI^2} \alpha_i\alpha_j \tag{$\lVert \Phi(\x) \rVert \leq 1$} \\
    &\leq \max_{\omega}| \balpha(\omega) |^2 \tag{Lemma 4 of~\citep{mutny2018efficient}}\\
    &\leq B^2.
\end{align}
Taking the square root gives us the final form.
\end{proof}
\subsection{Regret Bounds}
\begin{theorem}[$\beta_t$ concentration bound, Theorem~\ref{thm:beta} from the main paper]
Let $\lambda_{\min}, \lambda_{\max}$ and $\gamma$ be $(\alpha/2T)$
-accurate and regularizers $\bH_t \succcurlyeq 0 \ \forall t \in [T]$ are PSD. Let $\widehat\mu$ be a function in the RKHS $\cF_m$ that $\epsilon$-approximates $f \in \cH_k$ (Lemma~\ref{lem:existence_proximal_space}). Then, with probability at least $1-\alpha/2$, for all $\x \in \cD$ we have for all $t \in [T]$ simultaneously,
\begin{align}
    \left|\widehat\mu(\x) - \widetilde\mu_t(\x)\right| \leq \widetilde\sigma_t(\x)\underbrace{\left(\frac{B}{\rho}\sqrt{(\lambda_{\max} + \rho^2)} +\sqrt{2\log\frac{2}{\alpha} + \log\frac{\det\left(\bS_t + \rho^2 + \lambda_{\min}\bI\right)}{\det\left(\rho^2 + \lambda_{\min}\bI\right)}} + \frac{tB\epsilon}{\rho\sqrt{\lambda_{\min}}} + \frac{\kappa}{\rho}\right)}_{\beta^{1/2}_t}.
\end{align}
\label{thm:beta_app}
\end{theorem}
\begin{proof}
We wish to bound $\widehat\mu(\x) - \widetilde\mu_t(\x) = \left\langle \widehat\btheta - \widetilde\btheta_t, \Phi(\x)\right\rangle$. First, we bound this inner product by a suitable matrix norm:
\begin{align}
    \left\langle \widehat\btheta - \widetilde\btheta_t, \Phi(\x)\right\rangle &\leq \left\lVert \widehat\btheta - \widetilde\btheta_t \right\rVert_{\bV_t} \left\lVert \Phi(\x) \right\rVert_{\bV_t^{-1}} \\
    &= \frac{\widetilde\sigma_t(\x)}{\rho^2}\left\lVert \widehat\btheta - \widetilde\btheta_t \right\rVert_{\bV_t} \\
    &= \frac{\widetilde\sigma_t(\x)}{\rho^2}\left\lVert \widehat\btheta - \bV_t^{-1}\Phi(\bX_t)^\top\by_t - \bV_t^{-1}\bh_t \right\rVert_{\bV_t} \\
    &\leq \frac{\widetilde\sigma_t(\x)}{\rho^2}\left(\left\lVert \widehat\btheta - \bV_t^{-1}\Phi(\bX_t)^\top\by_t \right\rVert_{\bV_t} + \left\lVert \bh_t \right\rVert_{\bV_t^{-1}}\right). 
\end{align}
Now, let $\bz_t = (\langle \widehat\btheta, \phi(\x_\tau)\rangle + \varepsilon_\tau)_{\tau<=t}$. By Lemma~\ref{lem:existence_proximal_space}, we know that for each $|z_t^{i} - y_t^{i}| \leq B\epsilon$, and therefore $\lVert\bz_t - \by_t \rVert_2 \leq B\epsilon\sqrt{t}$. Using this fact:
\begin{align}
    \left\langle \widehat\btheta - \widetilde\btheta_t, \Phi(\x)\right\rangle & \leq \frac{\widetilde\sigma_t(\x)}{\rho^2}\left(\underbrace{\left\lVert \widehat\btheta - \bV_t^{-1}\Phi(\bX_t)^\top\bz_t \right\rVert_{\bV_t}}_{\Circled{A}} + \underbrace{\left\lVert \Phi(\bX_t)^\top\left(\bz_t - \by_t\right) \right\rVert_{\bV_t^{-1}}}_{\Circled{B}} + \underbrace{\left\lVert \bh_t \right\rVert_{\bV_t^{-1}}}_{\Circled{C}}\right).
\end{align}
Controlling $\Circled{A}$: Note that $\bz_t = \Phi(\bX_t)\widehat\btheta + \bvarepsilon_t$, and therefore $\Phi(\bX_t)^\top\bz_t = \Phi(\bX_t)^\top\Phi(\bX_t)\widehat\btheta + \Phi(\bX_t)^\top\bvarepsilon_t = \bV_t\widehat\btheta - \left(\bH_t + \lambda\bI\right)\widehat\btheta + \Phi(\bX_t)^\top\bvarepsilon_t$. Replacing this in $\Circled{A}$,
\begin{align}
    \left\lVert \widehat\btheta - \bV_t^{-1}\Phi(\bX_t)^\top\bz_t \right\rVert_{\bV_t} &= \left\lVert \widehat\btheta - \bV_t^{-1}\left(\bV_t\widehat\btheta - \left(\bH_t + \lambda\bI\right)\widehat\btheta + \Phi(\bX_t)^\top\bvarepsilon_t\right) \right\rVert_{\bV_t} \\
    &= \left\lVert  \left(\bH_t + \lambda\bI\right)\widehat\btheta + \Phi(\bX_t)^\top\bvarepsilon_t \right\rVert_{\bV_t^{-1}} \\
    &\leq \left\lVert  \left(\bH_t + \lambda\bI\right)\widehat\btheta\right\rVert_{\bV_t^{-1}} + \left\lVert \Phi(\bX_t)^\top\bvarepsilon_t \right\rVert_{\bV_t^{-1}} \\
    &\leq \left\lVert  \left(\bH_t + \lambda\bI\right)\widehat\btheta\right\rVert_{\left(\bH_t + \lambda\bI\right)^{-1}} + \left\lVert \Phi(\bX_t)^\top\bvarepsilon_t \right\rVert_{\bV_t^{-1}} \tag{$\bV_t \succcurlyeq \bH_t + \lambda\bI$}\\
    &\leq \left\lVert \widehat\btheta\right\rVert_{\bH_t + \lambda\bI} + \left\lVert \Phi(\bX_t)^\top\bvarepsilon_t \right\rVert_{\bV_t^{-1}} \tag{$\bV_t \succcurlyeq \bH_t + \lambda\bI$}\\
    &\leq \lVert \widehat\btheta \rVert_2\sqrt{\lambda_{\max} + \rho^2} + \left\lVert \Phi(\bX_t)^\top\bvarepsilon_t \right\rVert_{\bV_t^{-1}} \tag{$\bH_t \preccurlyeq \lambda_{\max}\bI$, union bound $\forall t \in [T]$ w. p. $1-\zeta/2$}\\
    &\leq B\sqrt{(\lambda_{\max} + \rho^2)} + \left\lVert \Phi(\bX_t)^\top\bvarepsilon_t \right\rVert_{\bV_t^{-1}} \tag{Lemma~\ref{lem:norm_bound_proximal_function}}\\
    &\leq B\sqrt{(\lambda_{\max} + \rho^2)} + \left\lVert \Phi(\bX_t)^\top\bvarepsilon_t \right\rVert_{\left(\bS_t + (\lambda + \lambda_{\min})\bI\right)^{-1}} \tag{$\bH_t \succcurlyeq \lambda_{\min}\bI$}
\end{align}
To bound the second term on the RHS, we use the ``self-normalized bound for vector-valued martingales'' of Abbasi-Yadkori~\etal\citep{abbasi2011improved} (Theorem 1), which gives us that with probability $1-\alpha/2$ for all $t \in [T]$ simultaneously,
\begin{align}
    \left\lVert \Phi(\bX_t)^\top\bvarepsilon_t \right\rVert_{\left(\bS_t + (\lambda + \lambda_{\min})\bI\right)^{-1}} \leq \rho\sqrt{2\log\frac{2}{\alpha} + \log\frac{\det\left(\bS_t + \lambda + \lambda_{\min}\bI\right)}{\det\left(\lambda + \lambda_{\min}\bI\right)}}.
\end{align}
Controlling $\Circled{B}$:
\begin{align}
    \lVert \Phi(\bX_t)^\top(\bz_t - \by_t) \rVert_{\bV_t^{-1}} &\leq \left\lVert \Phi(\bX_t)^\top\left(\bz_t - \by_t\right) \right\rVert_{\bH_t^{-1}} \tag{$\bV_t \succcurlyeq \bH_t$} \\
    &\leq \frac{\lVert \Phi(\bX_t)^\top(\bz_t - \by_t) \rVert_2}{\sqrt{\lambda_{\min}}} \tag{$\bH_t \succcurlyeq \lambda_{\min}\bI$} \\
    &\leq \frac{\left\lVert \Phi(\bX_t) \right\rVert_2\left\lVert\bz_t - \by_t\right\rVert_2}{\sqrt{\lambda_{\min}}} \tag{Cauchy-Schwarz} \\
    &\leq \frac{tB\epsilon}{\sqrt{\lambda_{\min}}}. \tag{$\lVert \Phi(\x) \rVert_2 \leq 1$ and definition of $\bz_t$}
\end{align}
Controlling $\Circled{C}$: We can see that $\lVert \bh_t \rVert_{\bV_t^{-1}} \leq \lVert \bh_t \rVert_{\bH_t^{-1}}$ and with probability $1-\zeta/2$ (from the control of $\Circled{A}$), for all rounds this is bounded by $\kappa$. 

Combining all three, we obtain that with probability at least $1-\zeta/2$, for any $\x \in \cD$, and simultaneously for all $t \in [T]$:
\begin{align}
    \widehat\mu(\x) - \widetilde\mu_t(\x) \leq \widetilde\sigma_t(\x)\underbrace{\left(\frac{B}{\rho}\sqrt{(\lambda_{\max} + \rho^2)} +\sqrt{2\log\frac{2}{\alpha} + \log\frac{\det\left(\widetilde\bS_t + \lambda_{\min}\bI\right)}{\det\left((\lambda + \lambda_{\min})\bI\right)}} + \frac{tB\epsilon}{\rho\sqrt{\lambda_{\min}}} + \frac{\kappa}{\rho}\right)}_{\beta^{1/2}_t}.
\end{align}
\end{proof}
\begin{lemma}[Variance Approximation]
\label{lem:sigma_bound}
Let $\sigma_t(\x) = \bk_t(\x)^\top(\bK_t + (\lambda + \lambda_{\min})\bI)^{-1}\bk_t(\x)$ where the quantities in $\bk_t$ and $\bK_t$ are determined by $k$ via Equation~\ref{eqn:original_update}. Then for all $t$ and $\x \in \cD$, we have that,
\begin{equation*}
    \widetilde\sigma_t(\x) \leq \sigma_t(\x) + \frac{2t^2\sqrt{\epsilon}}{\rho}.
\end{equation*}
\end{lemma}
\begin{proof}
First note that $\widetilde\sigma_t(\x) = \rho\lVert \Phi(\x) \rVert_{\bV_t^{-1}} \leq \rho\lVert \Phi(\x) \rVert_{(\bG_t + (\lambda +\lambda_{\min})\bI)^{-1}} = 1 - \widetilde\bk_t(\x)^\top(\widetilde\bK_t + (\lambda + \lambda_{\min})\bI)^{-1}\widetilde\bk_t(\x)$. Now, we will bound the second quantity on the RHS by $\sigma_t^2(\x)$.
\begin{align}
\widetilde\sigma_t(\x) &\leq \sigma_t(\x) +  \bk_t(\x)^\top(\bK_t + (\lambda + \lambda_{\min})\bI)^{-1}\bk_t(\x) -  \widetilde\bk_t(\x)^\top(\widetilde\bK_t + (\lambda + \lambda_{\min})\bI)^{-1}\widetilde\bk_t(\x).
\end{align}
Following identically the steps in Proposition 1 (by approximating the difference in terms of the Frobenius norm of $\widetilde\bK_t - \bK_t$ in terms of $\epsilon$) of Mutny \etal\citep{mutny2018efficient} (appendix), we obtain the remainder of the proof.
\end{proof}
\begin{theorem}[Regret Bound, Theorem~\ref{thm:main_regret_bound} from main paper]
Let $k$ be a stationary kernel with the associated RKHS $\cH_k$, and $\cF_m$ be an RKHS with feature $\Phi(\cdot)$ of dimensionality $m$, that $\epsilon$-uniformly approximates every $f \in \cH_k$ when $\lVert f \rVert \leq B$. Furthermore, assume $\lambda_{\min}, \lambda_{\max}$ and $\kappa$ such that they are $(\alpha/2T)$
-accurate and all regularizers $\bH_t \succcurlyeq 0 \ \forall t \in [T]$ are PSD. Then when the noise has variance $\rho^2$, \texttt{GP-UCB} with noisy quadrature Fourier features (NQFF) obtains the following cumulative regret with probability at least $1-\alpha$:
\begin{align*}
    \fR(T) \leq 2\left(B\sqrt{(\frac{\lambda_{\max}}{\rho^2} + 1)} +\sqrt{2\log\frac{2}{\alpha} + m\log(1+\frac{T}{\lambda+\lambda_{\min}})} + \frac{TB\epsilon}{\rho\sqrt{\lambda_{\min}}} + \frac{\kappa}{\rho}\right)\left( \sqrt{T\gamma_T} + \frac{T^3\sqrt{\epsilon}}{3\rho}\right) + 2TB\epsilon. 
\end{align*}
\end{theorem}
\begin{proof}
We first bound the instantaneous regret $r_t$ at any instant $t$.
\begin{align}
r_t &= f(\x_*) - f(\x_t) \\
&\leq \widehat\mu(\x_*) - \widehat\mu(\x_t) + 2B\epsilon \tag{Lemma~\ref{lem:existence_proximal_space}} \\
&\leq \beta_t\widetilde\sigma_t(\x_*) +\widetilde\mu_t(\x_*) - \widehat\mu(\x_t) + 2B\epsilon \tag{Theorem~\ref{thm:beta} ($\forall t \in [T]$ w.p. $\geq 1-\alpha/2$)} \\
&\leq \beta_t\widetilde\sigma_t(\x_t) +\widetilde\mu_t(\x_t) - \widehat\mu(\x_t) + 2B\epsilon \tag{Algorithm} \\
&\leq 2\beta_t\widetilde\sigma_t(\x_t) + 2B\epsilon \tag{Theorem~\ref{thm:beta} ($\forall t \in [T]$ w.p. $\geq 1-\alpha/2$)} \\
&\leq 2\beta_t\sigma_t(\x_t) + 2B\epsilon + \frac{2t^2\beta_t\sqrt{\epsilon}}{\rho}. \tag{Lemma~\ref{lem:sigma_bound}}
\end{align}
Now, we can sum over all rounds $t \in [T]$ to obtain the overall regret:
\begin{align}
    \fR(T) &= \sum_{t=1}^T r_t \leq 2\sum_{t=1}^T \left(\beta_t\sigma_t(\x_t) + B\epsilon + \frac{t^2\beta_t\sqrt{\epsilon}}{\rho}\right) \\
    &\leq 2\beta_T\left( \sum_{t=1}^T \sigma_t(\x_t) + \sqrt{\epsilon}\sum_{t=1}^T \frac{t^2}{\rho}\right) + 2TB\epsilon \\
    &\leq 2\beta_T\left( \sum_{t=1}^T \sigma_t(\x_t)\right) + \beta_T\frac{T^3\sqrt{\epsilon}}{3\rho} + 2TB\epsilon \\
    &\leq 2\beta_T\left( \sqrt{T\left(\sum_{t=1}^T \sigma^2_t(\x_t)\right)} + \frac{T^3\sqrt{\epsilon}}{3\rho}\right) + 2TB\epsilon \\
    &\leq 2\beta_T\left( \sqrt{T\log\left(\bI + (\lambda + \lambda_{\min})^{-1}\bK_T\right)} + \frac{T^3\sqrt{\epsilon}}{3\rho}\right) + 2TB\epsilon \tag{Lemma 5.4 of Srinivas~\etal\citep{srinivas2010gaussian} }\\
    &\leq 2\beta_T\left( \sqrt{T\gamma_T} + \frac{T^3\sqrt{\epsilon}}{3\rho}\right) + 2TB\epsilon \tag{Lemma 5.4 of Srinivas~\etal\citep{srinivas2010gaussian}}\\
    &= 2\left(\frac{B}{\rho}\sqrt{(\lambda_{\max} + \rho^2)} +\sqrt{2\log\frac{2}{\alpha} + \log\frac{\det\left(\bS_T + \lambda + \lambda_{\min}\bI\right)}{\det\left(\lambda+ \lambda_{\min}\bI\right)}} + \frac{TB\epsilon}{\rho\sqrt{\lambda_{\min}}} + \frac{\kappa}{\rho}\right)\left( \sqrt{T\gamma_T} + \frac{T^3\sqrt{\epsilon}}{3\rho}\right) + 2TB\epsilon
\end{align}
Further simplifying:
\begin{multline}
    = 2\left(\frac{B}{\rho}\sqrt{(\lambda_{\max} + \rho^2)} +\sqrt{2\log\frac{2}{\alpha} + \log\frac{\det\left(\Phi(\bX_T)^\top\Phi(\bX_T) + \lambda + \lambda_{\min}\bI\right)}{\det\left(\lambda + \lambda_{\min}\bI\right)}} + \frac{TB\epsilon}{\rho\sqrt{\lambda_{\min}}} + \frac{\kappa}{\rho}\right)\left( \sqrt{T\gamma_T} + \frac{T^3\sqrt{\epsilon}}{3\rho}\right) \\ + 2TB\epsilon 
\end{multline}
By the Hadamard inequality,
\begin{multline}
    \leq 2\left(\frac{B}{\rho}\sqrt{(\lambda_{\max} + \rho\textbf{}^2)} +\sqrt{2\log\frac{2}{\alpha} + \log\det\left(\frac{\text{diag}(\Phi(\bX_T)^\top\Phi(\bX_T))}{\lambda + \lambda_{\min}} + \bI\right)} + \frac{TB\epsilon}{\rho\sqrt{\lambda_{\min}}} + \frac{\kappa}{\rho}\right)\left( \sqrt{T\gamma_T} + \frac{T^3\sqrt{\epsilon}}{3\rho}\right) \\ + 2TB\epsilon 
\end{multline}
\begin{multline}
\label{eqn:regret_bound_final}
    \leq 2\left(B\sqrt{(\frac{\lambda_{\max}}{\rho^2} + 1)} +\sqrt{2\log\frac{2}{\alpha} + m\log(1+\frac{T}{\rho+\lambda_{\min}})} + \frac{TB\epsilon}{\rho\sqrt{\lambda_{\min}}} + \frac{\kappa}{\rho}\right)\left( \sqrt{T\gamma_T} + \frac{T^3\sqrt{\epsilon}}{3\rho}\right) + 2TB\epsilon. 
\end{multline}
\end{proof}
\begin{corollary}[Corollary~\ref{cor:regret_bound} from the main paper]
\label{cor:regret_bound_app}
Fix $m = (6\log T)^d$ and let $k$ be any kernel that obeys Assumption~\ref{ass:decomposability}. Algorithm~\ref{alg:main_algo} run with $m$-dimensional NQFF and noise $\bH_t, \bh_t$ that are $\alpha/2T$-accurate with constants $\lambda_{\max}, \lambda_{\min}$ and $\kappa$ obtains with probability at least $1-\zeta$, cumulative pseudoregret:
\begin{equation*}
\fR(T) = \cO\left(\sqrt{T\gamma_T}\left(\frac{B\sqrt{\lambda_{\max}}}{\rho} + \sqrt{\log\frac{1}{\zeta} + (\log T^6)^{d+1}} + \frac{\kappa}{\rho}\right)\right).
\end{equation*}
\end{corollary}
\begin{proof}
From Theorem~\ref{thm:qff_approx_error}, when $\bar{m} = 6 \log T$, and $\nu^2 < m$ we have that for any $\x, \by \in \cD,  \underset{\x, \by \in \cD}{\sup} |k(\x, \by) - \Phi(\x)^\top\Phi(\by)| \leq  \frac{C_1}{T^6} = \epsilon$ for some constant $C_1$. Replacing this in Equation~\ref{eqn:regret_bound_final} we see that the terms dependent on $\epsilon$ are $o(1)$, giving us the final result.
\end{proof}
\subsection{Privacy Bounds}
\begin{lemma}[Lemma~\ref{lem:jdp_noise} of the main paper]
Let $P_{m+1}$ be an $m-$dimensional multivariate Gaussian distribution with identity covariance $\sigma^2_{\alpha, \beta}\bI$. If $\sigma_{\alpha, \beta} > 16(1+\ceil{\log_2 T})(1+B^2+\rho^2(\ln 4T/\beta))\ln(4/\beta)^2\alpha^2$, then Alg.~\ref{alg:main_algo} with \textsc{Privatizer} following Alg.~\ref{alg:jdp_algo} is $(\alpha, \beta)-$jointly differentially private.
\end{lemma}
\begin{proof} 
First note that since $y_t$ is sub-Gaussian with mean at most $B$ (since $\lVert f \rVert_k \leq B$), we have from Lemma~\ref{lem:subgaussian_max_mean}, that with probability at least $1-\beta/4$, for each $(y_\tau)_{\tau \in [T]}$ simultaneously, 
\begin{align}
    |y_t|^2 \leq B^2 + 2\rho^2\log\frac{8T}{\beta}.
\end{align} 
The overall sensitivity $\Delta$ of each datum is then given by $\lVert \Phi(\x_t) \rVert_2 + |y_t|_2$, therefore $\Delta^2 \leq 1 + B^2 + 2\rho^2\log\frac{8T}{\beta}$ with probability at least $1-\beta/4$. Note now that we have the sum of at most $n = 1+ \ceil{\log T}$ noise variables. Therefore, to ensure $(\alpha, \beta)-$joint DP, we must ensure each noise variable preserves $(\alpha/\sqrt{8n\ln(2/\beta)}, \beta/2)$ privacy (based on zero-Concentrated DP~\citep{bun2016concentrated}). 

If $\sigma^2_{\alpha, \beta} > 16n(1 + B^2 + 2\rho^2\log\frac{8T}{\beta})^2\ln(\frac{10}{\beta})^2$ then we have by Lemma~\ref{lem:prob_gaussian_mechanism} that each noise term preserves $(\alpha/\sqrt{8n\ln(2/\beta)}, \beta/2)$-DP, proving the result.
\end{proof}

\begin{lemma}[Local JDP implies JDP, Lemma~\ref{lem:local_implies_jdp} from the main paper]
Any $(\alpha, \beta)-$local JDP algorithm $\cA$ protects $(\alpha,\beta)$-JDP for each $t \in [T]$.
\end{lemma}
\begin{proof}
Note that the output of the algorithm at any instant $t$ is merely $\x_t$, and the input data at any instant $t$ is $(\x_\tau, y_\tau)_{\tau < t}$. Therefore, we need to bound the ratio of probabilities for any two $t$-neighboring sequences $S$ and $S'$, for all $\tau \neq t$ and subset $\cS_{-t} = \cS_1 \times \cS_2 \times ... \times \cS_{t-1} \times \cS_{t+1} \times ... \times \cS_T \subset \cD_1 \times \cD_2 \times ... \times \cD_{t-1} \times \cD_{t+1} \times ... \times \cD_T$. Consider the actions taken under $S$ as $\x_1, ..., \x_{t-1}, \x_{t+1}, ..., \x_T$ and under $S'$ as $\x'_1, ..., \x'_{t-1}, \x'_{t+1}, ..., \x'_T$. Then, we have,
\begin{align}
    \frac{\bbP(\x_1, ..., \x_{t-1}, \x_{t+1}, ..., \x_T \in \cS_{-t})}{\bbP(\x'_1, ..., \x'_{t-1}, \x'_{t+1}, ..., \x'_T \in \cS_{-t})} &= \frac{\prod_{\tau = 1, \tau \neq t}^T \bbP(\x_\tau \in \cS_\tau | (\x_i, y_i)_{i=1}^\tau \in \cS_{<t})}{\prod_{\tau = 1, \tau \neq t}^T \bbP(\x'_\tau \in \cS_\tau | (\x'_i, y'_i)_{i=1}^\tau \in \cS_{<t})} \\ \intertext{Since $S$ and $S'$ only differ in $\cD_t$ and for identical subsequences up to instant $t$, $\cA$ is not stochastic. Therefore,}
    &= \frac{\prod_{\tau>t} \bbP(\x_\tau \in \cS_\tau | (\x_i, y_i)_{i=1}^\tau \in \cS_{<\tau})}{\prod_{\tau>t} \bbP(\x'_\tau \in \cS_\tau | (\x'_i, y'_i)_{i=1}^\tau \in \cS_{<\tau})}\\
    &= \frac{ \bbP(\x_{t+1} \in \cS_{t+1} | (\x_i, y_i)_{i=1}^{t+1} \in \cS_{<t+1})}{ \bbP(\x'_{t+1}\in \cS_{t+1} | (\x'_i, y'_i)_{i=1}^{t+1} \in \cS_{<t+1})} \\
    &\leq e^\alpha + \beta
\end{align}
Here, the last inequality follows from the fact that $S$ and $S'$ differ only in $\cD_t$ and that $\cA$ is $(\alpha, \beta)-$local JDP for all $t$.
\end{proof}
  \else
  \fi
\fi
\end{document}